\documentclass{article}
\usepackage[hyphens]{url}
\usepackage{hyperref}

\usepackage{graphicx}
\usepackage{hyperref}

\usepackage[accepted]{icml2020}
\usepackage{color}
\usepackage{comment}
\usepackage[utf8]{inputenc} 
\usepackage[T1]{fontenc}    
\usepackage{hyperref}       
\usepackage{url}            
\usepackage{booktabs}       
\usepackage{amsfonts}       
\usepackage{nicefrac}       
\usepackage{microtype}      
\usepackage{physics}
\usepackage{amsmath,amsfonts,amssymb,amsthm,epsfig,epstopdf,url,array,algorithm,booktabs}
\usepackage{algcompatible}
\usepackage[font=small,labelfont=bf]{caption}
\newcommand\numberthis{\addtocounter{equation}{1}\tag{\theequation}}
\theoremstyle{plain}
\newtheorem{thm}{Theorem}[section]
\newtheorem{lem}[thm]{Lemma}

\newtheorem*{cor}{Corollary}
\usepackage{xspace}
\usepackage{caption}
\usepackage{subcaption}
\theoremstyle{definition}
\newtheorem{defn}{Definition}[section]

\newtheorem{innercustomthm}{Theorem}
\newenvironment{customthm}[1]
  {\renewcommand\theinnercustomthm{#1}\innercustomthm}
  {\endinnercustomthm}
\theoremstyle{remark}

\usepackage{multirow}
\newcommand{\squishlist}{
	\begin{list}{$\bullet$}
		{
			\setlength{\itemsep}{0pt}
			\setlength{\parsep}{3pt}
			\setlength{\topsep}{3pt}
			\setlength{\partopsep}{0pt}
			\setlength{\leftmargin}{1.5em}
			\setlength{\labelwidth}{1em}
			\setlength{\labelsep}{0.5em} } }
	
\newcommand{\squishend}{
\end{list}  }\newcommand{\squishlistnum}{
	\begin{enumerate}
		{
			\setlength{\itemsep}{0pt}
			\setlength{\parsep}{3pt}
			\setlength{\topsep}{3pt}
			\setlength{\partopsep}{0pt}
			\setlength{\leftmargin}{1.5em}
			\setlength{\labelwidth}{1em}
			\setlength{\labelsep}{0.5em} } }
	
\newcommand{\squishendnum}{
\end{enumerate}  }

\usepackage{wrapfig}
\DeclareMathAlphabet{\mathpzc}{OT1}{pzc}{m}{it}

\begin{document}

\twocolumn[
\icmltitle{Data-Dependent Differentially Private Parameter Learning for Directed Graphical Models}
\begin{icmlauthorlist}
\icmlauthor{Amrita Roy Chowdhury}{to}
\icmlauthor{Theodoros Rekatsinas}{to}
\icmlauthor{Somesh Jha}{to,goo}

\end{icmlauthorlist}

\icmlaffiliation{to}{Department of Computer Sciences, University of Wisconsin-Madison, USA}
\icmlaffiliation{goo}{XaiPient, USA}

\icmlcorrespondingauthor{Amrita Roy Chowdhury}{amrita@cs.wisc.edu}

\icmlkeywords{Machine Learning, ICML}
\vskip 0.3in

]

\printAffiliationsAndNotice{}

\begin{abstract}

Directed graphical models (DGMs) are a class of probabilistic models that are widely used for predictive analysis in sensitive domains such as medical diagnostics. In this paper, we present an algorithm for differentially private learning of the parameters of a DGM. Our solution optimizes for the utility of inference queries over the DGM and \textit{adds noise that is customized to the properties of the private input dataset and the graph structure of the DGM}. To the best of our knowledge, this is the first explicit data-dependent privacy budget allocation algorithm in the context of DGMs. We compare our algorithm with a standard data-independent approach over a diverse suite of  benchmarks and demonstrate that our solution requires a privacy budget that is roughly $3\times$ smaller to obtain the same or higher utility.

\end{abstract}

\section{Introduction}\label{intro}

Directed graphical models (DGMs) are  a class of probabilistic models that are widely used in causal reasoning and  predictive analytics 
~\cite{PGMbook}. A typical use case for these models is  answering ``what-if'' queries over domains that often work with sensitive information. For example, DGMs are used in medical diagnosis for answering questions, such as what is the most probable disease given a set of symptoms~\cite{DGM1}. In learning such models, it is common that the underlying graph structure of the model is publicly known. For instance, in the case of medical data, the dependencies between several physiological symptoms and diseases are well established, standardized, and publicly available. However, the parameters of the model have to be learned from observations. These observations may contain sensitive information as in the case of medical applications. Hence, learning and publicly releasing the parameters of the probabilistic model may lead to privacy violations~\cite{attack1, attack2}, and thus, the need for privacy-preserving learning mechanisms for DGMs.

In this paper, we focus on the problem of privacy-preserving learning of the parameters of a DGM. For our privacy definition, we use differential privacy (DP)~\cite{dwork} -- currently the de-facto standard for privacy. We consider the setting when \textit{the structure of the target DGM is publicly known and the parameters of the model are learned from fully observed data}. In this case, all parameters can be estimated via counting queries over the input observations (also referred to as {\em data set} in the remainder of the paper). The direct way to ensure differential privacy is to add suitable noise to the observations using the standard Laplace mechanism~\cite{dwork}. Unfortunately, this method is data-independent, i.e., the noise added to the base observations is oblivious of the properties of the input data set and the structure of the DGM, resulting in sub-optimal utility. To address this issue, we turn to {\em data-dependent methods} which add noise that is customized to the properties of the input data sets~\cite{DAWA,AHP,data1,hist,hist2,DPcube,Kotsogiannis:2017:PDD:3035918.3035945}.

We propose a {\em data-dependent}, $\epsilon$-DP algorithm for learning the parameters of a DGM over fully observed data. Our goal is to minimize  errors in arbitrary inference queries that are subsequently answered over the learned DGM. The main contributions are:

(1) \textbf{Explicit data-dependent privacy-budget allocation:} Our algorithm computes the parameters of the conditional probability distribution of each random variable in the DGM via separate measurements from the input data set. This lets us optimize the privacy budget allocation across the different variables with the objective of reducing the error in inference queries. We formulate this optimization objective in a data-dependent manner -- our optimization objective is informed by both the private input data set and the public graph structure of the DGM. To the best of our knowledge, this is the first work to propose \textit{explicit data-dependent privacy-budget allocation in the context of DGMs}.  We evaluate our algorithm on four DGM benchmarks and demonstrate that our scheme only requires a privacy budget of $\epsilon=1.0$ to yield the same utility that a data-independent baseline achieves with a much higher $\epsilon=3.0$. Specifically, our baseline is based on \cite{Bayes4} which is the most recent work that explicitly deals with differentially private parameter estimation for DGMs. 
 
(2)\textbf{ New theoretical results:} To preserve privacy, we add noise to the parameters of the DGM. To understand how this noise propagates to inference queries, we provide two new theoretical results on the upper and lower bound of the error of inference queries. The upper bound has an exponential dependency on the treewidth of the DGM while the lower bound depends on its maximum degree. We also provide a formulation to compute the sensitivity \cite{sensitivity1} of the parameters associated with a node of a DGM targeting the probability distribution of its child nodes only. To the best of our knowledge, these theoretical results are novel.

\section{Background}\label{sec:background}
In this section, we review basic background material relevant to this paper.

\textbf{Directed Graphical Models:} A directed graphical model (DGM) or a Bayesian network is a probabilistic model that is represented as a directed acyclic graph, \scalebox{0.9}{$\mathcal{G}$}. \begin{wrapfigure}{r}{0.5\columnwidth}
\hspace{-0.2cm}\centering\includegraphics[width=0.45\columnwidth,height=2.7cm]{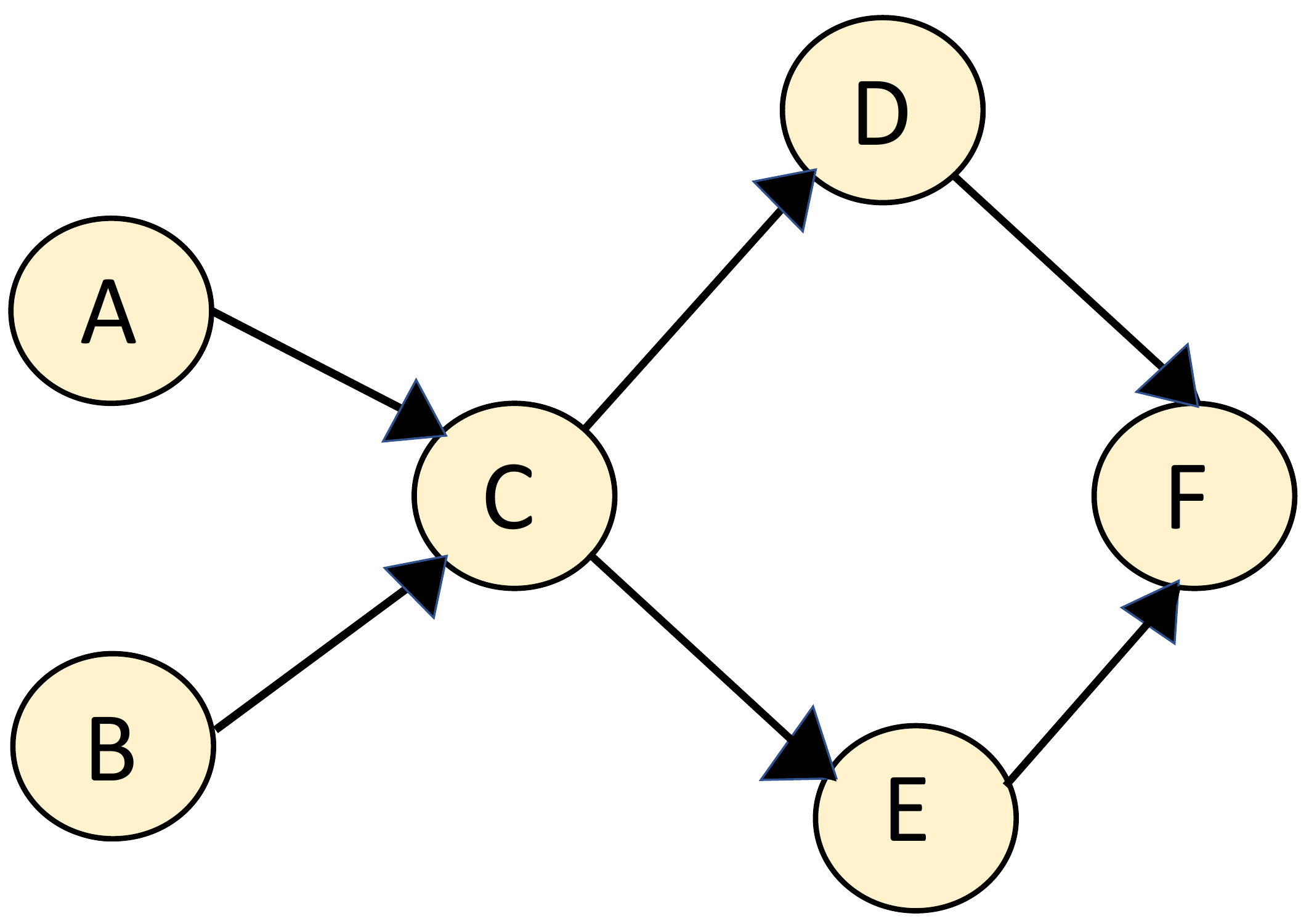}
\caption{An example directed graphical model} \label{fig:DGM}
\end{wrapfigure}The nodes of the graph represent random variables and the edges encode conditional dependencies between the variables. The graphical structure of the DGM represents a factorization of the joint probability distribution of these random variables. Specifically, given a DGM with graph \scalebox{0.9}{$\mathcal{G}$}, let \scalebox{0.9}{$X_{1},\ldots  ,X_{n}$} be the random variables corresponding  to the nodes of \scalebox{0.9}{$\mathcal{G}$} and \scalebox{0.9}{$X_{pa_{i}}$} denote the set of parents in \scalebox{0.9}{$\mathcal{G}$} for the node corresponding to variable \scalebox{0.9}{$X_i$}.  The joint probability distribution factorizes as
\begin{gather}\scalebox{0.9}{$P[X_{1},\ldots ,X_{n}]=\prod _{i=1}^{n}P[X_{i}|X_{pa_{i}}]$} \end{gather}
where each factor \scalebox{0.9}{$P[X_{i}|X_{pa_{i}}]$} corresponds to a conditional probability distribution (CPD). For example, for the DGM depicted by Fig. \ref{fig:DGM}, we have \scalebox{0.87}{$P[A,B,C,D,E,F]=P[A]\cdot P[B]\cdot P[C|A,B]\cdot P[D|C]\cdot P[E|C]\cdot$} \scalebox{0.87}{$P[F|D,E]$.} For DGMs with discrete random variables, each CPD can be represented as a table of parameters \scalebox{0.9}{$\Theta_{x_i|x_{pa_i}}$} where each parameter corresponds to a conditional probability and \scalebox{0.9}{$x_i$} and \scalebox{0.9}{$x_{pa_i}$} denote variable assignments \scalebox{0.9}{$X_i=x_i$} and \scalebox{0.9}{$X_{pa_i}=x_{pa_i}$}. 
  
A key task in DGMs is \textbf{parameter learning}. Given a DGM with a graph structure \scalebox{0.9}{$\mathcal{G}$}, the goal of parameter learning is to estimate each \scalebox{0.9}{$\Theta_{x_i|x_{pa_i}},$} a task solved via maximum likelihood estimation (MLE). In the presence of fully observed data \scalebox{0.9}{$\mathcal{D}$} (i.e., data corresponding to all the nodes of \scalebox{0.9}{$\mathcal{G}$}\footnote{The attributes of the data set become the nodes of the DGM's graph. For the remainder of the paper we use them interchangeably.} is available), the maximum likelihood estimates of the CPD parameters take the closed-form~\cite{PGMbook}
\begin{equation}
\small
\Theta_{x_i|x_{pa_i}}=C[x_i,x_{pa_i}]/C[x_{pa_i}]
\label{eq:learning}
\end{equation}
where \scalebox{0.9}{$C[x_i]$} is the number of records in $\mathcal{D}$ with \scalebox{0.9}{$X_i=x_i$}.

After learning, the DGM is used to answer {\bf inference queries}, i.e., queries that compute the probabilities of certain events (variables) of interest. Inference queries can also include evidence (a subset of the nodes has a fixed assignment). There are three types inference queries in general:  

(1)\textbf{ Marginal inference:} This is used to answer queries of the type "what is the probability of a given variable if all others are marginalized". An example marginal inference query  for the DGM in Fig.\hspace{0.1cm}\ref{fig:DGM} is \scalebox{0.9}{$P[F=0]=\sum_{A}\sum_{B}\sum_{C}\sum_{D}\sum_{F}P[A,B,C,D,E,F=0]$}.

(2)\textbf{ Conditional Inference:} This type of query answers the probability distribution of some variable conditioned on some evidence $e$. An example conditional inference query for the DGM in Fig. \ref{fig:DGM} is \scalebox{0.9}{$P[A|F=0]=$}$\frac{P[A,F=0]}{P[F=0]}$. 

(3)\textbf{ Maximum a posteriori (MAP) inference:} This type of query asks for the most likely assignment of variables. An example MAP query for the DGM in Fig.~\ref{fig:DGM} is
\scalebox{0.9}{$\max_{A,B,C,D,E}\{P[A,B,C,D,E,F=0]\}$}.

For DGMs, inference queries can be answered exactly by the {\bf variable elimination (VE)} algorithm~\cite{PGMbook} which is described in detail in Appx. \ref{app:DGM}. The basic idea is that we "\textit{eliminate}" one variable at a time following a predefined order \scalebox{0.9}{$\prec$} over the graph nodes. Let \scalebox{0.9}{$\Phi$} denote a set of probability factors \scalebox{0.9}{$\phi$} (initialized with all the CPDs of the DGM) and \scalebox{0.9}{$Z$} denote the variable to be eliminated. First, all probability factors involving $Z$ are removed from \scalebox{0.9}{$\Phi$} and multiplied together to generate a new product factor. Next, \scalebox{0.9}{$Z$} is summed out from this combined factor, generating a new factor \scalebox{0.9}{$\boldsymbol{\phi}$} that is entered into \scalebox{0.9}{$\Phi$}. Thus, VE corresponds to repeated sum-product computations: \scalebox{0.9}{$\boldsymbol{\phi}=\sum_Z\prod_{\phi \in \Phi} \phi$}.

 Additionally, we define a term Markov blanket which is used in Sec.~\ref{sec:Alg:desc}.
\begin{defn}[Markov Blanket]
The Markov blanket, denoted by \scalebox{0.9}{$\mathpzc{P}(X)$}, for a node \scalebox{0.9}{$X$} in a graphical model is the set of all nodes such that given \scalebox{0.9}{$\mathpzc{P}(X)$}, \scalebox{0.9}{$X$} is conditionally independent of all the other nodes. \cite{markovblanket}. \label{MB}\vspace{-0.2cm}\end{defn}
In a DGM, the Markov blanket of a node consists of its child nodes, parent nodes, and the parents of its child nodes. For example, in Fig. \ref{fig:DGM}, \scalebox{0.9}{$\mathcal{P}(D)=\{C,E,F\}$.}

\textbf{Differential Privacy:} We formally define differential privacy (DP) as  follows: \begin{defn}[Differential Privacy] \label{def:dp}
A randomized algorithm \scalebox{0.9}{$\mathcal{A}$}
satisfies $\epsilon$-differential privacy ($\epsilon$-DP), where $\epsilon > 0$ is a privacy parameter, iff
 for any two data sets \scalebox{0.9}{$\mathcal{D}$} and \scalebox{0.9}{$\mathcal{D}'$} that differ in a single record, we have
\begin{equation} 
\small
\forall t \in Range(\mathcal{A}), P\big[\mathcal{A}(\mathcal{D}) = t\big] \leq e^{\epsilon}P\big[\mathcal{A}(\mathcal{D}') = t\big] \label{eq:DP}
\end{equation} 
\end{defn}In our setting, $\mathcal{A}$ (in Eq. \eqref{eq:DP}) corresponds to an algorithm for learning the parameters of a DGM with a publicly known graph structure from a fully observed data set $\mathcal{D}$. 

When applied multiple times, the DP guarantee degrades gracefully as follows.

\begin{thm}[Sequential Composition] \label{theorem:seq}
If \scalebox{0.9}{$\mathcal{A}_1$} and
\scalebox{0.9}{$\mathcal{A}_2$} are $\epsilon_1$-DP and $\epsilon_2$-DP algorithms that use independent randomness, then releasing the outputs \scalebox{0.9}{$(\mathcal{A}_1(\mathcal{D}),\mathcal{A}_2(\mathcal{D}))$} on database \scalebox{0.9}{$\mathcal{D}$} satisfies $(\epsilon_1+\epsilon_2)$-DP.\end{thm} 
Any post-processing computation performed on the noisy output of a DP algorithm does not degrade privacy.
\begin{thm}[Post-processing]\label{theorem:post}
Let \scalebox{0.9}{$\mathcal{A}: \mathcal{D} \mapsto R$} be a $\epsilon$-DP algorithm. Let \scalebox{0.9}{$f : R \mapsto R'$} be an
arbitrary randomized mapping. Then \scalebox{0.9}{$f \circ \mathcal{A} : \mathcal{D} \mapsto R'$} is $\epsilon$-DP. \vspace{-0.2cm}\end{thm}
The privacy guarantee of a DP algorithm can be amplified by a preceding sampling step \cite{PAClearning,Amplification}.
Let \scalebox{0.9}{$\mathcal{A}$} be an $\epsilon$-DP algorithm and \scalebox{0.9}{$\mathcal{D}$} be a data set. Let \scalebox{0.9}{$\mathcal{A}'$} be an algorithm that runs \scalebox{0.9}{$\mathcal{A}$} on a random subset of  \scalebox{0.9}{$\mathcal{D}$} obtained by sampling it with probability \scalebox{0.9}{$\beta$}.
\begin{lem}[Privacy Amplification] Algorithm \scalebox{0.9}{$\mathcal{A}'$} will satisfy $\epsilon'$-DP where \scalebox{0.9}{$\epsilon'=ln(1+\beta(e^\epsilon -1))$} \label{sampling}\vspace{-0.2cm}
\end{lem}
The \textbf{Laplace mechanism} is a standard algorithm to achieve differential privacy \cite{dwork}. In this mechanism, in order
to output \scalebox{0.9}{$f(\mathcal{D})$} where \scalebox{0.9}{$f : \mathcal{D} \mapsto R$}, an $\epsilon$-DP algorithm $\mathcal{A}$ 
publishes \scalebox{0.9}{$f(\mathcal{D}) + Lap\Big(\frac{\Delta f}{\epsilon}\Big)$ } 
where \scalebox{0.9}{$\Delta f = \max_{\mathcal{D},\mathcal{D}'}||f(\mathcal{D})-f(\mathcal{D}')||_1$} is known as the sensitivity of the function. The probability density function of \scalebox{0.9}{$Lap(b)$} is given by \scalebox{0.9}{$\mathbf{f}(x)={\frac{1}{2b}}e^{ \left(-{\frac  {|x-\mu |}{b}}\right)}$}. The sensitivity of the function \scalebox{0.9}{$f$} is the maximum magnitude by which an individual's data can change \scalebox{0.9}{$f$}. 
The sensitivity of counting queries is 1. 

Next, we define two terms, namely marginal table and mutually consistent marginal tables, that are used in Sec. \ref{sec:Alg:desc}.

Let \scalebox{0.9}{$\mathcal{D}$} be a data set defined over attributes \scalebox{0.9}{$\mathcal{X}$} and \scalebox{0.9}{$\mathbf{X}$} be an attribute set such that \scalebox{0.9}{$\mathbf{X}=\{X_1,\cdots,X_k\}, \mathbf{X}\subseteq \mathcal{X}$}. Let \scalebox{0.9}{$s=\prod_{i=1}^k |dom(X_i)|$} and \scalebox{0.9}{$dom(\mathbf{X)}=\{v_1, \cdots , v_s\}$} represent the domain of \scalebox{0.9}{$\mathbf{X}$}. The \textbf{marginal table} for the attribute set \scalebox{0.9}{$\mathbf{X}$} denoted by \scalebox{0.9}{$M_\mathbf{X}$}, is computed as follows:

(1) Populate the entries of table \scalebox{0.9}{$T_\mathbf{X}$} of size $s$ from $\mathcal{D}$ such that each entry \scalebox{0.9}{$j \in [s], T_\mathbf{X}[j]=\#$} records in \scalebox{0.9}{$\mathcal{D}$} with \scalebox{0.9}{$\mathbf{X}=v_j$}.
This step is also called \textit{materialization}.

(2) Compute \scalebox{0.9}{$M_\mathbf{X}$} from \scalebox{0.9}{$T_\mathbf{X}$} such that \scalebox{0.9}{$M_\mathbf{X}[j]=T_\mathbf{X}[j]/\sum_{i=1}^sT_\mathbf{X}[i], j \in [s]$}.

Thus the entries of \scalebox{0.9}{$M_\mathbf{X}$} corresponds to the values of the joint probability distribution over the attributes in \scalebox{0.9}{$\mathbf{X}$}.

Let \scalebox{0.9}{$Attr(M)$} denote the set of attributes on which a marginal table \scalebox{0.9}{$M$} is defined and \scalebox{0.9}{$M_1 \equiv M_2$}
denote that the two marginal tables have the same values for
every entry.
\begin{defn}[Mutually Consistent Marginal Tables] \label{def:mutualconsistency}Two noisy
marginal tables \scalebox{0.9}{$\tilde{M}_i$} and \scalebox{0.9}{$\tilde{M}_j$} are defined to be mutually consistent iff the
marginal table over the attributes in \scalebox{0.9}{$Attr(\tilde{M}_i) \cap Attr(\tilde{M}_j)$} reconstructed from
\scalebox{0.9}{$\tilde{M}_i$} is exactly the same as the one reconstructed from  
\scalebox{0.9}{$\tilde{M}_j$}, i.e.,  \small \begin{gather} \tilde{M}_i[Attr(\tilde{M}_i) \cap Attr(\tilde{M}_j)] \equiv \tilde{M}_j[Attr(\tilde{M}_i) \cap Attr(\tilde{M}_j)]
\end{gather} \end{defn}

\section{Data-Dependent Differentially Private Parameter Learning for DGMs}\vspace{-0.1cm}
In this section, we describe our proposed solution for differentially private learning of the parameters of a fully observed DGM by adding data and structure dependent noise.  
\subsection{Problem Setting}

Let \scalebox{0.9}{$\mathcal{D}$} be a sensitive data set of size $m$ with attributes  \scalebox{0.9}{$\mathcal{X}=\langle X_1,\cdots,X_n\rangle$} and let \scalebox{0.9}{$\mathcal{N}=\langle\mathcal{G},\Theta\rangle$} be the DGM of interest. The graph structure \scalebox{0.9}{$\mathcal{G}$} of \scalebox{0.9}{$\mathcal{N}$ }defined over the attribute set \scalebox{0.9}{$\mathcal{X}$} is publicly known. Our goal is to \textit{learn the parameters \scalebox{0.9}{$\Theta$}, i.e., the CPDs of \scalebox{0.9}{$\mathcal{N}$}, in a data-dependent differentially private manner from \scalebox{0.9}{$\mathcal{D}$} such that the error in inference queries over the $\epsilon$-DP DGM is minimized}. 
\subsection{Key Ideas}
Our solution is based on the following two key observations:

(1) The parameters \scalebox{0.9}{$\Theta[X_i|X_{{pa}_i}]$} of the DGM \scalebox{0.9}{$\mathcal{N}$} can be estimated separately via \textit{counting queries over the empirical marginal table} of the attribute set  \scalebox{0.9}{$X_i \cup X_{pa_i}$}. 

(2) The \textit{factorization over \scalebox{0.9}{$\mathcal{N}$} decomposes the overall $\epsilon$-DP learning problem into a set of separate $\epsilon$-DP learning sub-problems (one for each CPD)}. For example, for the DGM in Fig. \ref{fig:DGM}, the following six CPDs have to be learned separately \scalebox{0.9}{$\{P[A],P[B],P[C|A,B],P[D|C],P[E|C],$} \scalebox{0.9}{$P[F|D,E]\}$}. Thus the total privacy budget has to be divided among these sub-problems. However, due to the structure of the graph and the data set, some nodes will have more impact on inference queries than others. Hence, allocating more budget (and thus, getting better accuracy) to these nodes will result in reduced overall error for the inference queries. 

Our method is outlined in Alg. \ref{algo:main} and proceeds in two stages. In the first stage, we obtain preliminary noisy measurements of the parameters of \scalebox{0.9}{$\mathcal{N}$} which are used along with some graph specific properties (the height and out-degree of each node) to formulate a data-dependent optimization objective for privacy budget allocation. The solution of this objective is then used in the second stage to compute the final parameters. For instance, for the DGM in Fig. \ref{fig:DGM}, node $A$ (root node) would typically have a higher privacy budget than node $F$ (leaf node). In summary, if \scalebox{1}{$\epsilon^B$} is the total privacy budget available, we spend \scalebox{1}{$\epsilon^I$} to obtain preliminary parameter measurements in Stage I and the remaining \scalebox{1}{$\epsilon^B-\epsilon^I$} is used for the final parameter computation in Stage II, after optimal allocation across the marginal tables. As a result, our scheme only requires a privacy budget of $\epsilon=1.0$ to yield the same utility that a standard data-independent method achieves with $\epsilon=3.0$ (Sec. \ref{sec:evaluation}).

Next, we describe our algorithm in detail and highlight how we address the two core technical challenges in our solution: 

(1) how to reduce the privacy budget cost for the first stage $\epsilon^I$ (equivalently increase $\epsilon^B$-$\epsilon^I$) (Alg. \ref{algo:main}, Lines 1-3), and  

(2) what properties of the data set and the graph should the optimization objective be based on (Alg. \ref{algo:main}, Lines 5-11).
\begin{algorithm}[tb]\label{algo:main}
\small{

\caption{Differentially private learning of the parameters of a directed graphical model}
\begin{algorithmic}[1]
\STATEx \textbf{Input:} $\mathcal{D}$\hspace{0.09cm} - Input data set of size $m$ with attributes
\STATEx \hspace{1.55cm}\scalebox{0.8}{$\mathcal{X}=\langle X_1,\cdots,X_n\rangle$};
\STATEx \hspace{0.9cm}
 $\mathcal{G} $ \hspace{0.09cm}- Graph Structure of DGM  $\mathcal{N}$;
\STATEx \hspace{0.9cm} $\epsilon^B$- Total privacy budget; \STATEx \hspace{0.9cm} $\beta$\hspace{0.09cm}  - Sampling rate for Stage I; \STATEx \hspace{0.9cm}$\epsilon^I$ \hspace{0.09cm}- Privacy budget for Stage I
\STATEx \textbf{Output:} $\widetilde{\Theta}$ - Noisy parameters of $\mathcal{N}$   \vspace{0.1cm}
\Statex \textbf{Stage I:}  Optimization objective formulation for privacy\Statex \hspace{1.1cm} budget allocation
\vspace{0.1cm}
\STATE  $\epsilon = \ln\big(\frac{e^{\epsilon^I}-1}{\beta}+1\big)$ \hspace{0.8cm} \textcolor{blue}{$\rhd$} Computing privacy parameter  \STATE Construct a new dataset $\mathcal{D}'$ by sampling $\mathcal{D}$ with probability $\beta$
\STATE  $\mathpzc{E}=[\frac{\epsilon}{n},\cdots,\frac{\epsilon}{n}]$
\STATE \scalebox{0.9}{$\widehat{\Theta}, \widehat{T}, \widehat{T}_{pa}=$}\vspace{0.1cm} \scalebox{1}{$ComputeParameters$}\scalebox{0.9}{$(\mathcal{D}',\mathpzc{E})$}

\STATE \textbf{for} $i = 1$ to $n$ 
\STATE \hspace{0.4cm}  $\widetilde{{\delta}_i}$ \hspace{0.1cm}= \scalebox{1}{$ComputeError$}\scalebox{0.9}{$(i,\widehat{T}_i,\widehat{T}_{pa_i})$} \hspace{0.05cm} \vspace{0.1cm}\Statex \hspace{0.7cm}\textcolor{blue}{$\rhd$} Estimating error of parameters for $X_i$ using Eq.~\eqref{eq:delta} \vspace{0.1cm}
\STATE \hspace{0.45cm}$h_i$ \hspace{0.1cm}= Height of node $X_i$\vspace{0.1cm}
\STATE  \hspace{0.4cm} $o_i$ \hspace{0.1cm}= Out-degree of node $X_i$\vspace{0.1cm}
\STATE \hspace{0.35cm}  \scalebox{0.9}{$\widetilde{\Delta}^{\mathcal{N}}_i$} = \scalebox{1}{$ComputeSensitivity$}\scalebox{0.9}{$(i, \widehat{\Theta})$}  \vspace{0.1cm} \Statex \hspace{0.7cm}\textcolor{blue}{$\rhd$} Computing sensitivity of the parameters using Eq. \eqref{sensitivity1}
\vspace{0.1cm}
\STATE \hspace{0.45cm}  \scalebox{0.9}{$W[i]= (h_i+1)\cdot (o_i+1)\cdot (\widetilde{\Delta}^{\mathcal{N}}_i+1)$}
\STATE \textbf{end for}
\STATE  \scalebox{0.9}{$\mathcal{F}_{\mathcal{G},\mathcal{D}}=\sum_{i=1}^{n-1} W[i]\cdot \widetilde{\delta_i}/{\epsilon_i} + W[n]\cdot\widetilde{\delta_n}/(\epsilon^B-\epsilon^I-\sum_{i=1}^{n-1}\epsilon_i)$} \vspace{0.1cm}\Statex \hspace{4cm} \textcolor{blue}{$\rhd$} Optimization Objective
\vspace{0.1cm}
\Statex \textbf{Stage II:}  Final computation of the parameter \scalebox{0.9}{$\Theta$}\vspace{0.1cm}
 \STATE Solve for  \scalebox{0.9}{$\mathpzc{E}^*=\{\epsilon^*_i\}$} from minimizing \scalebox{0.9}{$\mathcal{F}_{\mathcal{G},\mathcal{D}}$}  using Eq. \eqref{sol}\vspace{0.05cm}
 \STATE  \scalebox{0.9}{$\overline{\Theta},\overline{T}, \overline{T}_{pa} = $} \scalebox{1}{$ComputeParameters$}\scalebox{0.9}{$(\mathcal{D},\mathpzc{E}^*)$}
 \STATE \scalebox{0.9}{$\widetilde{\Theta}=\varnothing
 $}
 \STATE \textbf{for} \scalebox{0.9}{$X_i \in \mathcal{X}$}\STATEx \textcolor{blue}{$\rhd$} Assuming \scalebox{0.8}{$\widehat{\Theta}=\underset{X_i \in \mathcal{X}}{\bigcup}\widehat{\Theta}[X_i|X_{pa_i}]$} and \scalebox{0.8}{$\overline{\Theta}=\underset{X_i \in \mathcal{X}}{\bigcup}\overline{\Theta}[X_i|X_{pa_i}]$} \STATE \hspace{0.43cm} $\widehat{\epsilon_i} = (\epsilon^I/n)/(\mathpzc{E}^*[i]+\epsilon^I/n), \overline{\epsilon_i}=\mathpzc{E}^*[i]/(\mathpzc{E}^*[i]+\epsilon^I/n)$\vspace{0.2cm}
 \STATE \hspace{0.43cm} \scalebox{0.9}{$\widetilde{\Theta}[X_i|X_{pa_i}] = 
\widehat{\epsilon_i}\cdot\widehat{\Theta}[X_i|X_{pa_i}] + \overline{\epsilon_i}\cdot\overline{\Theta}[X_i|X_{pa_i}]$} \Statex \hspace{5.5cm}\textcolor{blue}{$\rhd$}Weighted Mean
 \STATE \hspace{0.43cm} \scalebox{0.9}{$\widetilde{\Theta}=\widetilde{\Theta}\bigcup \widetilde{\Theta}[X_i|X_{pa_i}]$}
 \STATE \textbf{end for}
 \STATE Return  $\widetilde{\Theta}$
\end{algorithmic}\label{algo:main}}
\end{algorithm}
\setlength{\textfloatsep}{10pt}
\begin{algorithm}[tb]\label{procedure}
\small{
\caption*{\textbf{Procedure 1} \scalebox{0.95}{$ComputeParameters$}}\label{algo:compute_para}
\begin{algorithmic}[1]
\STATEx \textbf{Input:}   \scalebox{0.9}{$\mathcal{D}$} - Data set with attributes \scalebox{0.9}{$\mathcal{X}=\langle X_1,\cdots,X_n\rangle$};\Statex \hspace{0.8cm} \scalebox{0.9}{$\mathpzc{E}[1,\cdots,n]$} - Array of privacy budget allocation
\STATE \textbf{for}  \scalebox{0.9}{$i = 1$} to  $n$
\STATE \hspace{0.3cm} Materialize tables  \scalebox{0.9}{$T_i$} and \scalebox{0.9}{$T_{pa_i}$} for the attribute sets \STATEx \hspace{0.3cm} \scalebox{0.9}{$X_i \cup X_{pa_i}$} and \scalebox{0.9}{$X_{pa_i}$} respectively for  \scalebox{0.9}{$\mathcal{D}$} \vspace{0.1cm}
\STATE \hspace{0.3cm}Add noise \scalebox{0.9}{$\sim Lap(\frac{2}{\mathpzc{E}[i]})$} to each entry of \scalebox{0.9}{ $T_i$} and \scalebox{0.9}{$T_{pa_i}$} \STATEx \hspace{0.3cm} to generate  noisy tables \scalebox{0.9}{$\widetilde{T_i}$} and \scalebox{0.9}{$\widetilde{T}_{pa_i}$} respectively \vspace{0.1cm}
\STATE \hspace{0.3cm} Convert \scalebox{0.9}{$\widetilde{T_i}$} into noisy marginal table  \scalebox{0.9}{${\widetilde{M}_i}$} using \scalebox{0.9}{$\widetilde{T}_{pa_i}$}
\STATE \textbf{end for}
\STATE \scalebox{1}{$MutualConsistency$}\scalebox{0.9}{$(\bigcup_{X_i \in \mathcal{X}} \widetilde{M}_i$)} \vspace{0.1cm} \Statex \hspace{0.5cm}\textcolor{blue}{$\rhd$}Ensures mutual consistency (Def. \ref{def:mutualconsistency}) among the noisy \Statex \hspace{0.7cm} marginal tables sharing
subsets of attributes
\vspace{0.1cm}
\STATE \textbf{for} \scalebox{0.9}{$i=1$} to $n$
\STATE \hspace{0.3cm} Construct  \scalebox{0.9}{$\widetilde{\Theta}[X_i|X_{pa_i}]$} from \scalebox{0.9}{$\widetilde{M_i}$} \hspace{1.5cm}\textcolor{blue}{$\rhd$} Using Eq. \ref{eq:learning}
\STATE \textbf{end for}
\vspace{0.1cm}
\STATE Return  \scalebox{0.9}{$\widetilde{\Theta}=\underset{X_i \in \mathcal{X}}{\bigcup}\widetilde{\Theta}{[X_i|X_{pa_i}]}, \widetilde{T}=\underset{X_i \in \mathcal{X}}{\bigcup}\widetilde{T}_i, \widetilde{T}_{pa}=\underset{X_i \in \mathcal{X}}{\bigcup}\widetilde{T}_{pa_i}$}
\end{algorithmic}}
\end{algorithm}

 \subsection{Algorithm Description} \label{sec:Alg:desc} 
 We now describe the two stages of our technique:
 
 \textbf{Stage I -- Formulation of optimization objective:} First, we handle the trade-off between the two parts of the total privacy budget \scalebox{0.9}{$\epsilon^I$} and \scalebox{0.9}{ $\epsilon^B-\epsilon^I$}. While we want to maximize \scalebox{0.9}{$\epsilon^B-\epsilon^I$} to reduce the amount of noise in the final parameters, sufficient budget \scalebox{0.9}{$\epsilon^I$} is required to obtain good estimates of the statistics of the data set to form the data-dependent budget allocation objective. To handle this trade-off, we use the sampling strategy from Lemma~\ref{sampling} to improve the accuracy of the optimization objective computation (Alg. \ref{algo:main}, Lines 1-2). This allows us to assign a relatively low value to \scalebox{0.9}{$\epsilon^I$} increasing our budget for the final parameter computation. 

Next, we estimate the parameters \scalebox{0.9}{$\widehat{\Theta}$} on the sampled data set \scalebox{0.9}{$\mathcal{D}'$} via the procedure \scalebox{0.95}{$ComputeParameters$} (described below and outlined in Procedure 1) using budget allocation \scalebox{0.9}{$\mathpzc{E}$}  (Alg. \ref{algo:main}, Lines 3-4). 
Note that \scalebox{0.9}{$\widehat{\Theta}$} is only required for the optimization objective formulation and is different from the final parameters \scalebox{0.9}{$\widetilde{\Theta}$} (Alg. \ref{algo:main}, Line 18). Hence, for \scalebox{0.9}{$\widehat{\Theta}$} we use a naive allocation policy of equal privacy budget for all tables. 

Finally, we compute the privacy budget optimization objective \scalebox{0.9}{$\mathcal{F}_{\mathcal{D},\mathcal{G}}$} that depends on the data set \scalebox{0.9}{$\mathcal{D}$} and graph structure \scalebox{0.9}{$\mathcal{G}$} (Alg. \ref{algo:main}, Line 5-12). The details are discussed in Sec. \ref{optimization}.  

\textbf{Stage II -- Final parameter computation:} We solve for the optimal privacy budget allocation \scalebox{1}{$\mathpzc{E}^*$} from \scalebox{0.9}{$\mathcal{F}_{\mathcal{D},\mathcal{G}}$} and use it to compute a copy of the parameters \scalebox{0.9}{$\overline{\Theta}$} (Alg. \ref{algo:main}, Lines 13-14). We obtain the final parameters \scalebox{0.9}{$\widetilde{\Theta}$} by computing the weighted average of the corresponding values in \scalebox{0.9}{$\overline{\Theta}$} and the preliminary estimate \scalebox{0.9}{$\widehat{\Theta}$} (Alg. \ref{algo:main}, Lines 15-20). Note that \scalebox{0.9}{$\mathpzc{E}^*[i]+\epsilon^I/n$} is the total privacy budget spent on the CPD of node \scalebox{0.9}{$X_i$} in the two rounds.

\textbf{Procedure 1 \scalebox{1}{$ComputeParameters$}:} The goal of this procedure is, \textit{given a privacy budget allocation \scalebox{0.9}{$\mathpzc{E}$}, to derive the parameters of \scalebox{0.9}{$\mathcal{N}$} under DP}. First, we materialize the tables for the attribute sets \scalebox{0.9}{$ X_i \cup X_{pa_i}$} and \scalebox{0.9}{$X_{pa_i} i \in [n] $} (Proc. 1, Line 2), and then inject noise drawn from \scalebox{0.9}{$Lap(\frac{2}{\mathpzc{E}[i]})$} (using half of the privacy budget \scalebox{0.9}{$\frac{\mathpzc{E}[i]}{2}$} for each table) into each of their cells (Proc. 1, Line 3) to generate \scalebox{0.9}{$\widetilde{T}_i$} and \scalebox{0.9}{$\widetilde{T}_{pa_i}$} respectively. 
Next, we convert \scalebox{0.9}{$\widetilde{T}_i$} and \scalebox{0.9}{$\widetilde{T}_{pa_i}$} to a marginal table \scalebox{0.9}{$\widetilde{M}_i$}, i.e., joint distribution  \scalebox{0.9}{$P_{\mathcal{N}}[X_i,X_{pa_i}]$} (Proc. 1, Line 4) as
\begin{gather*}\vspace{-4cm}\small \scalebox{1.05}{$\widetilde{M_i}[x_i,x_{pa_i}]=$}\scalebox{1.2}{$\frac{T[x_i,x_{pa_i}]/T_{pa_i}[x_{pa_i}]}{\underset{v\in dom(X_i)}{\sum}\hspace{-0.3cm}T[v,x_{pa_i}]/T_{pa_i}[x_{pa_i}]}$}\vspace{-1cm}
\end{gather*} The denominator in the above equation is for normalization. This is followed by ensuring that all \scalebox{0.9}{$\widetilde{M}_i$}s are mutually consistent (Def. \ref{def:mutualconsistency}) on all the attribute subsets (Proc. 1, Line 6). For this, we follow the techniques outlined in \cite{Consistency,PriView} and further described in Appx.~\ref{consistency}. Finally, we derive \scalebox{0.9}{$\widetilde{\Theta}[X_i|X_{pa_i}]$} (the noisy estimate of \scalebox{0.9}{$P_{\mathcal{N}}[X_i|X_{pa_i}]$)} from  \scalebox{0.9}{$\widetilde{M_i}$} (Proc. 1, Lines 7-10). Note that although \scalebox{0.9}{$\widetilde{M}_i$} could have been derived from \scalebox{0.9}{$\widetilde{T}_i$} alone, we also use \scalebox{0.9}{$\widetilde{T}_{pa_i}$} for its computation to ensure independence of the added noise in Eq. \eqref{eq:expected_error}. 

\subsection{Optimal Privacy Budget Allocation}\label{optimization}  

Our goal is to find the optimal privacy budget allocation over the marginal tables, \scalebox{0.9}{$\widetilde{M}_i, i \in [n]$} for \scalebox{0.9}{$\mathcal{N}$} such that the error in the subsequent inference queries on \scalebox{0.9}{$\mathcal{N}$} is minimized.

\textbf{{Observation I:}} \textit{A more accurate estimate of the parameters of $\mathcal{N}$ will result in better accuracy for the subsequent inference queries}. Hence, we focus on reducing the total error of the parameters of \scalebox{0.9}{$\mathcal{N}$}. From Eq. \eqref{eq:learning} and our Laplace noise injection (Proc. 1, Line 3), for a privacy budget of $\epsilon$, the value of a parameter of the DGM computed from the noisy marginal tables \scalebox{0.9}{$\widetilde{M}_i$} is expected to be
\begin{gather*}\vspace{-0.5cm}
\scalebox{0.9}{$
\widetilde{\Theta}[x_i|x_{pa_i}]=\Big(\hspace{-0.4cm}\underbrace{C[x_i,x_{pa_i}]}_{\substack{\text{True count for records}\\\text{with $X_i=x_i$}\\\text{ and $X_{pa_i}=x_{pa_i}$}}}\hspace{-0.2cm}\pm\hspace{-0.1cm}\underbrace{\frac{2\sqrt{2}}{\epsilon}}_{\substack{\text{Noise due to}\\\text{Laplace}\\\text{ mechanism}}}\hspace{-0.3cm}\Big)\Big/\Big(C[x_{pa_i}]\pm$}\scalebox{1.2}{$\frac{2\sqrt{2}}{\epsilon}\Big)$} \numberthis \label{eq:expected_error}
\end{gather*} 
Thus, from the rules of standard error propagation \cite{error}, the error in \scalebox{0.9}{$\Theta[x_i,x_{pa_i}]$}  is 
\begin{gather*}\small\vspace{-0.3cm}
\hspace{-0.2cm}\delta_{\Theta[x_i,x_{pa_i}]}\scalebox{0.85}{$=\Theta[x_i,x_{pa_i}] \sqrt{8/(\epsilon \cdot C[x_{pa_i}])^2+8/(\epsilon \cdot C[x_i,x_{pa_i}])^2}$} 
\numberthis\label{eq:delta}\end{gather*} 
where \scalebox{0.8}{$C[x_i]$} denotes the number of records in \scalebox{0.9}{$\mathcal{D}$} with  \scalebox{0.9}{$X_i=x_i$}.
Hence, the mean error for the parameters of  \scalebox{0.9}{$X_i$} is  \begin{gather*}\vspace{-0.2cm}\small\delta_{i}\scalebox{1.2}{$=\frac{1}{|dom(X_i \cup X_{pa_i})|}\sum_{x_i,x_{pa_i}} \delta_{\Theta[x_i|x_{pa_i}]}$}\vspace{-0.2cm}\end{gather*} 
where \scalebox{0.9}{$dom(S)$} is the domain of the attribute set \scalebox{0.9}{$S$}. Since using the true counts, \scalebox{0.9}{$C[x_i]$}, would violate privacy, Alg. \ref{algo:main} uses the noisy estimates from  \scalebox{0.9}{$\widetilde{T_i}$} and \scalebox{0.9}{$\widetilde{T}_{pa_i}$} (Alg. 
\ref{algo:main}, Line 6).

\textbf{{Observation II:}} \textit{ Depending on the data set and the graph structure, different nodes will have different impact on inferencing}. This information can be captured by a corresponding weighting coefficient  \scalebox{0.9}{$W[i], i \in [n]$} for each node. 


\textbf{Computation of weighting coefficient  \scalebox{0.9}{$W[i]$}:} For a given node \scalebox{0.9}{$X_i$}, the weighting coefficient \scalebox{0.9}{$W[i]$} is computed from the following three node features: 

(1) \textbf{Height of the node $h_i$:}
The height of a node \scalebox{0.9}{$X_i$} is the length of the longest path between \scalebox{0.9}{$X_i$} and a leaf node. Due to the factorization of the joint distribution over a DGM, the marginal probability distribution of a node depends only on the set of its ancestor nodes (as is explained in the following discussion on the computation of sensitivity). Thus, a node with large height will affect the inference queries on more nodes (all its successors) than say a leaf node.  

(2)\textbf{ Out-degree of the node $o_i$:} A  node causally affects all its children nodes. Thus the impact of a node with high out-degree on inferencing will be more  than say a leaf node.

(3)\textbf{ Sensitivity \scalebox{0.9}{$\Delta^{\mathcal{N}}_i$}:} Sensitivity of a parameter in a DGM  measures the impact of  small changes in the parameter value on a target probability. Laskey~\cite{sensitivity1} proposed a method  of computing sensitivity by using the partial derivative of output probabilities with respect to the parameter being varied. However, previous works have mostly focused on the target probability to be a joint distribution of all the variables. In this paper, we present a method to compute sensitivity by targeting the probability distribution of child nodes only.
Let  \scalebox{0.9}{$\Delta^{\mathcal{N}}_i$} denote the mean sensitivity of the parameters of  \scalebox{0.9}{$X_i$} on target probabilities of all the nodes in  \scalebox{0.9}{$Child(X_i)$}= $\{$set of all the child nodes of  $X_i\}$. Formally,
\begin{gather*}
\hspace{-2cm}
\scalebox{0.95}{$ \Delta^{\mathcal{N}}_i=$}
\scalebox{1}{$ \frac{1}{|dom(X_i\cup X_{pa_i})|}\Big(\underset{x_i,x_{pa_i}}{\sum}\frac{1}{|Child(X_i)|}\big($}
\\
\hspace{1.8cm}
\underbrace{\scalebox{1}{$\underset{Y \in Child(X_i)}{\sum}\hspace{-0.1cm}\frac{1}{|dom(Y)|}(\sum_y\pdv {P_{\mathcal{N}}[Y=y]}{\Theta[x_i|x_{pa_i}]})\big)\Big)$}}_{\substack{\text{computing the partial derivatives}
\\
\text{of the parameters of the child nodes only}}}
\hspace{-0.4cm}
\numberthis\label{sensitivity1}\end{gather*} 
A node  \scalebox{0.9}{$X_i$} can affect another node  \scalebox{0.9}{$Y$} only iff  it is in its Markov blanket (Defn. \ref{MB}), i.e., \scalebox{0.9}{$Y \in \mathpzc{P}(X_i)$}. However due to the factorization of the joint distribution over a DGM,  \scalebox{0.9}{$\forall Y \in \mathpzc{P}(X_i), Y \not \in Child(X_i), P_{\mathcal{N}}[Y]$} can be expressed without  \scalebox{0.9}{$\Theta[x_i|x_{pa_i}]$}.  Thus just computing the mean sensitivity of the parameters over the set of child nodes \scalebox{0.9}{$\Delta^{\mathcal{N}}_i$} turns out to be a good weighting metric for our setting.  \scalebox{0.9}{$\Delta^{\mathcal{N}}_i$} for leaf nodes is thus 0. Note that  \scalebox{0.9}{$\Delta^{\mathcal{N}}_i$} is distinct from the notion of sensitivity of a function in the Laplace mechanism (Sec. \ref{sec:background}).

\textbf{Computing  Sensitivity \scalebox{0.9}{$\Delta^{\mathcal{N}}_i$}:} Let  \scalebox{0.9}{$Y \in Child(X_i)$} and  \scalebox{0.9}{$\Gamma(Y)=\{\mathbf{Y}_1,\cdots,\mathbf{Y}_t\}, t < n$} denote the set of all nodes such that there is a directed path from \scalebox{0.9}{ $\mathbf{Y}_i$} to  \scalebox{0.9}{$Y$}. In other words \scalebox{0.9}{$\Gamma(Y)$} denotes the set of ancestors of  \scalebox{0.9}{$Y$} in  \scalebox{0.9}{$\mathcal{G}$}. From the factorization of the joint distribution over \scalebox{0.9}{$\mathcal{N}$}, we have 
\begin{gather*} \small \scalebox{0.95}{$P_{\mathcal{N}}[\mathbf{Y}_1,\cdots,\mathbf{Y}_t,Y]=P_{\mathcal{N}}[Y|Y_{pa}]\cdot \prod_{\mathbf{Y}_i \in \Gamma(Y)} P_{\mathcal{N}}[\mathbf{Y}_i|\mathbf{Y}_{pa_i}]$ } \\\scalebox{0.95}{$P_{\mathcal{N}}[Y]=\sum_{\mathbf{Y}_1}...\sum_{\mathbf{Y}_t} P_{\mathcal{N}}[\mathbf{Y}_1,\cdots,\mathbf{Y}_t,Y]$}\end{gather*}
Therefore, using our noisy preliminary parameter estimates (Alg. \ref{algo:main}, Line 4), we compute \begin{gather*} \hspace{-0.5cm}\scalebox{0.9}{${\pdv{\widetilde{P}_{\mathcal{N}}[Y=y]}{\Theta[x_i|x_{pa_i}]}}=\hspace{-0.6cm}\underset{\substack{\mathbf{y}_i \in dom(\mathbf{Y}_i),\\ \mathbf{y}_{pa_i} \in dom(\mathbf{Y}_{{pa}_i}),\\ \mathbf{Y}_i\in \Gamma(Y)}}{\sum}\hspace{-0.3cm}\Big(
\underset{\mathbf{Y}_i\in \Gamma(Y)}{\prod}\widehat{\Theta}[\mathbf{Y}_i=\mathbf{y}_i|\mathbf{Y}_{{pa}_i}\mathbf{y}_{pa_i}]\boldsymbol{\cdot}$} \\ \scalebox{0.9}{\hspace{-0.2cm}$\underset{x_i,x_{pa_i}}{\zeta}\hspace{-0.3cm}(\mathbf{Y}_i=\mathbf{y}_i,\mathbf{Y}_{pa_i}\hspace{-0.1cm}=\mathbf{y}_{pa_i})\Big) \boldsymbol{\cdot} \widehat{\Theta}[Y\hspace{-0.1cm}=y|Y_{pa}\hspace{-0.1cm}=y_{pa}] \boldsymbol{\cdot} \hspace{-0.4cm}\underset{x_i,x_{pa_i}}{\zeta}\hspace{-0.4cm}(Y_{pa}=y_{pa})$}\end{gather*}\begin{equation*}\scalebox{0.9}{$ \underbrace{\underset{x_i,x_{pa_i}}{\zeta} \hspace{-0.3cm}(Z_1=z_1,\cdots,Z_t=z_t)}_{\substack{\text{indicator variable to ensure}\\\text{only relevant terms are retained}}} = \hspace{-0.1cm}\left\{
                \begin{array}{ll}
                 \hspace{-0.2cm}1 \hspace{0.15cm}\mbox{ if } \underset{i=1}{\bigcup}^t Z_i \bigcap \{X_i \cup X_{pa_i}\}\hspace{-0.1cm}=\hspace{-0.05cm}\varnothing   \\ \hspace{-0.2cm}1 \hspace{0.15cm}
                  \mbox{ if } \forall Z_i, Z_i \in \{X\cup X_{pa_i}\}\hspace{-0.05cm}\Rightarrow\\\hspace{2.1cm} z_i \in \{x,x_{pa_i}\}  \\
                  \hspace{-0.2cm
                  }0 \hspace{0.15cm} \mbox{ otherwise}
                \end{array}
              \right. $\label{sensitivity2}\numberthis} \end{equation*} where  \scalebox{0.9}{$ \zeta_{x_i,x_{pa_i}}$} is an indicator variable which ensures that only the product terms  \scalebox{0.9}{$\prod_{\mathbf{Y}_i\in \Gamma(Y)}\widehat{\Theta}[y_i|y_{pa_i}]$} involving parameters with attributes in  \scalebox{0.9}{$\{X_i,X_{pa_i},Y\}$} that match up with the corresponding values in  \scalebox{0.9}{$\{x_i,x_{pa_i},y\}$}  are retained in the computation (as all others terms  have partial derivative 0).
Thus, the noisy mean sensitivity estimate for the parameters of node \scalebox{0.9}{$X_i$},  \scalebox{0.9}{$\widetilde{\Delta}^{\mathcal{N}}_i$}, can be computed  from Eq. \eqref{sensitivity1} and \eqref{sensitivity2}.



\textbf{Optimization Objective:} Let $\epsilon_i$ denote the privacy budget for node \scalebox{0.9}{$X_i$}. Thus from the above discussion, the optimization objective  \scalebox{0.9}{$\mathcal{F}_{\mathcal{D},\mathcal{G}}$} is formulated as a weighted sum of the parameter error and the optimization problem is given by\
\begin{gather*}\underset{\epsilon_i}{\mbox{\bf{minimize} }}\thinspace\thinspace\thinspace \scalebox{0.9}{$\mathcal{F}_{\mathcal{D},\mathcal{G}}$}\\ \hspace{1.2cm}\mbox{ \bf{subject to }} \thinspace\thinspace\thinspace\scalebox{0.9}{$\epsilon_i>0, \forall i \in [n]  $}\numberthis\label{error}\\
 \scalebox{0.9}{$\mathcal{F}_{\mathcal{D},\mathcal{G}} = \Big( \sum_{i=1}^{n-1} \underbrace{W[i]}_{\text{weight}}\cdot \underbrace{\frac{\widetilde{\delta}_i}{\epsilon_i}}_{\text{mean error}} \hspace{-0.2cm}+ \thinspace W[n]\cdot$}\scalebox{1}{$ \frac{\widetilde{\delta}_{n}}{\epsilon^B-\epsilon^I-\sum_{i=1}^{n-1}\epsilon_i}\Big)$}\numberthis
\label{objective} \\ \hspace{-2cm}\scalebox{0.9}{\mbox{}}
 \scalebox{0.87}{$ W[i]=(\underbrace{h_i}_{\text{height}}+1)\cdot (\underbrace{o_i}_{\text{out-degree}}+1)\cdot (\underbrace{\widetilde{\Delta}^{\mathcal{N}}_i}_{\text{sensitivity}}+1)$}\numberthis\label{W}\\\hspace{-3.5cm}\scalebox{0.9}{$\widetilde{\delta}_i=$}\scalebox{1}{$\frac{1}{|dom(X_i\cup X_{pa_i})|}
 \sum_{x_i,x_{pa_i}}\Big($}\\\hspace{1.6cm}\scalebox{0.9}{$\widehat{\Theta}[x_i|x_{pa_i}]\sqrt{1/\widehat{T}[x_{pa_i}]^2+1/\widehat{T}[x_i,x_{pa_i}]^2}$}\numberthis\label{eq:delta}\Big) \\ \hspace{-1.3cm}\scalebox{0.9}{$\widehat{T}[x_i,x_{pa_i}],\widehat{T}[x_{pa_i}] \in \widehat{T}_i$}\end{gather*}
 where $h_i$  and $o_i$ are the height and out-degree of the node  \scalebox{0.9}{$X_i$} respectively, \scalebox{0.9}{$\Delta^{\mathcal{N}}_i$} is the sensitivity of the parameters of   \scalebox{0.9}{$X_i$}, $\frac{\widetilde{\delta}_i}{\epsilon_i} $ gives the measure for estimated mean error for the parameters (CPD) of  \scalebox{0.9}{$X_i$} and the denominator of the last term of Eq. \eqref{objective} captures the linear constraint  \scalebox{0.9}{$\sum_{i=1}^n\epsilon_i=\epsilon^B-\epsilon^I $}. As stated by Eq. \eqref{W}, the weighting coefficient \scalebox{0.9}{$W[i]$} is defined as the product of the aforementioned three features. The extra additive term $1$ is used to handle leaf nodes to ensure non-zero weighting coefficients. Let $\epsilon^*_i$ denote the optimal privacy budget for node \scalebox{0.9}{$X_i$}. The objective \scalebox{0.9}{$\mathcal{F}_{\mathcal{D},\mathcal{G}}$} has a closed form solution as follows
\begin{gather*}\vspace{-0.8cm}\small 
 \scalebox{1}{$c_j=1/(W[j]\cdot \widetilde{\delta}_j), j \in [n], \epsilon^{II}=\epsilon^B-\epsilon^I$} \\\hspace{-0.2cm}\scalebox{1}{$\epsilon^*_i=\frac{\epsilon^{II}\hspace{-0.1cm}\underset{j=1, j\neq i}{\prod^n}\hspace{-0.2cm}\sqrt{c_j}}{{\sum}_{j \in [n]}\underset{l=1, l \neq j}{\prod^n}\hspace{-0.2cm}\sqrt{c_l}}, i \hspace{-0.1cm}\in\hspace{0cm} [n-1],\epsilon^*_n=\hspace{-0.05cm}\epsilon^{II}-\hspace{-0.1cm} \overset{n-1}{\underset{i=1}{\sum}}\epsilon^*_i$} \numberthis\label{sol}  \end{gather*}  
\textbf{Discussion:} There are two sources of information to be considered for a DGM - (1) graph structure  \scalebox{0.9}{$\mathcal{G}$} (2) data set  \scalebox{0.9}{$\mathcal{D}$}. $h_i$ and $o_i$ are purely graph characteristics that summarise the graphical properties of the node  \scalebox{0.9}{$X_i$}.  \scalebox{0.9}{$\widetilde{\Delta}^{\mathcal{N}}_i$} captures the interactions of the graph structure with the actual parameter values thereby encoding the data set dependent information. Hence, we theorize that the aforementioned three features are sufficient for constructing the weighting coefficients.

Also note that it is trivial to modify our proposed algorithm to allow estimation with Dirichlet priors (the most popular choice for a prior \cite{PGMbook}). Specifically, R.H.S of Eq. \eqref{eq:learning} changes to \scalebox{0.9}{$(C[x_i,x_{pai}]+\alpha_k)/(C[x_{pai}] +\sum_k(\alpha_k)$} where \scalebox{0.9}{${\alpha_k}$} are the parameters of the publicly known prior.

\textbf{Illustration of  Algorithm \ref{algo:main}:} Here we illustrate Alg.\hspace{0.1cm}\ref{algo:main} on the example DGM of Fig.\hspace{0.1cm}\ref{fig:DGM}. The parameters \scalebox{0.9}{$\Theta$} of this DGM are the CPDs  \scalebox{0.9}{$\{P[A],P[B],P[C|A,B],P[D|C],P[E|C]$}, \scalebox{0.9}{$P[F|D,E]\}$}. Thus we need to construct $6$  marginal tables over the attribute sets  \scalebox{0.9}{$\langle\{A\},\{B\},\{C,A,B\},\{D,C\},$} \scalebox{0.9}{$\{E,C\},\{F,D,E\}\rangle$}. First, we compute a preliminary estimate of the above parameters from a sampled dataset  \scalebox{0.9}{$\mathcal{D}'$} (Alg. \ref{algo:main}, Line 1-4). For this, we need to ensure mutual consistency between  \scalebox{0.9}{$\widetilde{M}_A$} and  \scalebox{0.9}{$\widetilde{M}_C$} on attribute  \scalebox{0.9}{$A$},  \scalebox{0.9}{$\widetilde{M}_B$} and  \scalebox{0.9}{$\widetilde{M}_C$} on attribute  \scalebox{0.9}{$B$} and so on (Proc. 1, Line 6). This is followed by the formulation of  \scalebox{0.9}{$\mathcal{F}_{\mathcal{D},\mathcal{G}}$}. (Alg. \ref{algo:main}, Line 5-12). Here we show the computation of  \scalebox{0.9}{$W[i]$} for node  \scalebox{0.9}{$A$}. For simplicity, we assume binary attributes. \scalebox{0.9}{$h_A=3$} and \scalebox{0.9}{$o_A=1$} trivially. For \scalebox{0.9}{$\Delta^{\mathcal{N}}_A$}, we need to compute the sensitivity of the parameters of \scalebox{0.9}{$A$} on the target probability of \scalebox{0.9}{$C$}, i.e., \scalebox{0.9}{$\pdv{P[C=0]}{\Theta[A=0]},\pdv{P[C=0]}{\Theta[A=1]}, \pdv{P[C=1]}{\Theta[A=1]}$},  and  \scalebox{0.9}{$\pdv{P[C=1]}{\Theta[A=0]}$} which is computed as \scalebox{0.9}{$ \pdv{\widetilde{P}[C=0]}{\Theta[A=0]}=$}\scalebox{0.9}{$\widehat{\Theta}[B=0]\widehat{\Theta}[C=0|A=0,B=0]+$}
\scalebox{0.9}{$\widehat{\Theta}[B=1]\widehat{\Theta}[C=0|A=0,B=1] $}. The rest of the partial derivatives are computed in a similar manner to give us  \begin{gather*}\scalebox{0.95}{ $ \Delta^{\mathcal{N}}_A=\frac{1}{4}\Big(\pdv{P[C=0]}{\Theta[A=0]}$}\scalebox{0.95}{ $+\pdv{P[C=0]}{\Theta[A=1]}+\pdv{P[C=1]}{\Theta[A=1]}+\pdv{P[C=1]}{\Theta[A=0]}\Big) $}\end{gather*} Finally, we use the solution of  \scalebox{0.9}{$\mathcal{F}_{\mathcal{D},\mathcal{G}}$} to compute the final parameters (Alg. \ref{algo:main}, Line 13-21).  

\subsection{Privacy Analysis}
\begin{thm} The proposed algorithm (Alg. \ref{algo:main}) for learning the parameters of a DGM with a publicly known graph structure over fully observed data is $\epsilon^B$-DP. \vspace{-0.2cm} \end{thm}
The proof of the above theorem follows from Thms. \ref{theorem:seq} and \ref{theorem:post} and is presented in Appx.~\ref{Privacy proof}. The DGM learned via our algorithm can be released publicly and any inference query run on it will still be $\epsilon^B$-DP (Thm. \ref{theorem:post}).

\section{Error Analysis for Inference Queries}
\begin{figure*}[ht]
    \begin{subfigure}[b]{0.5\columnwidth}
    \includegraphics[width=\columnwidth]{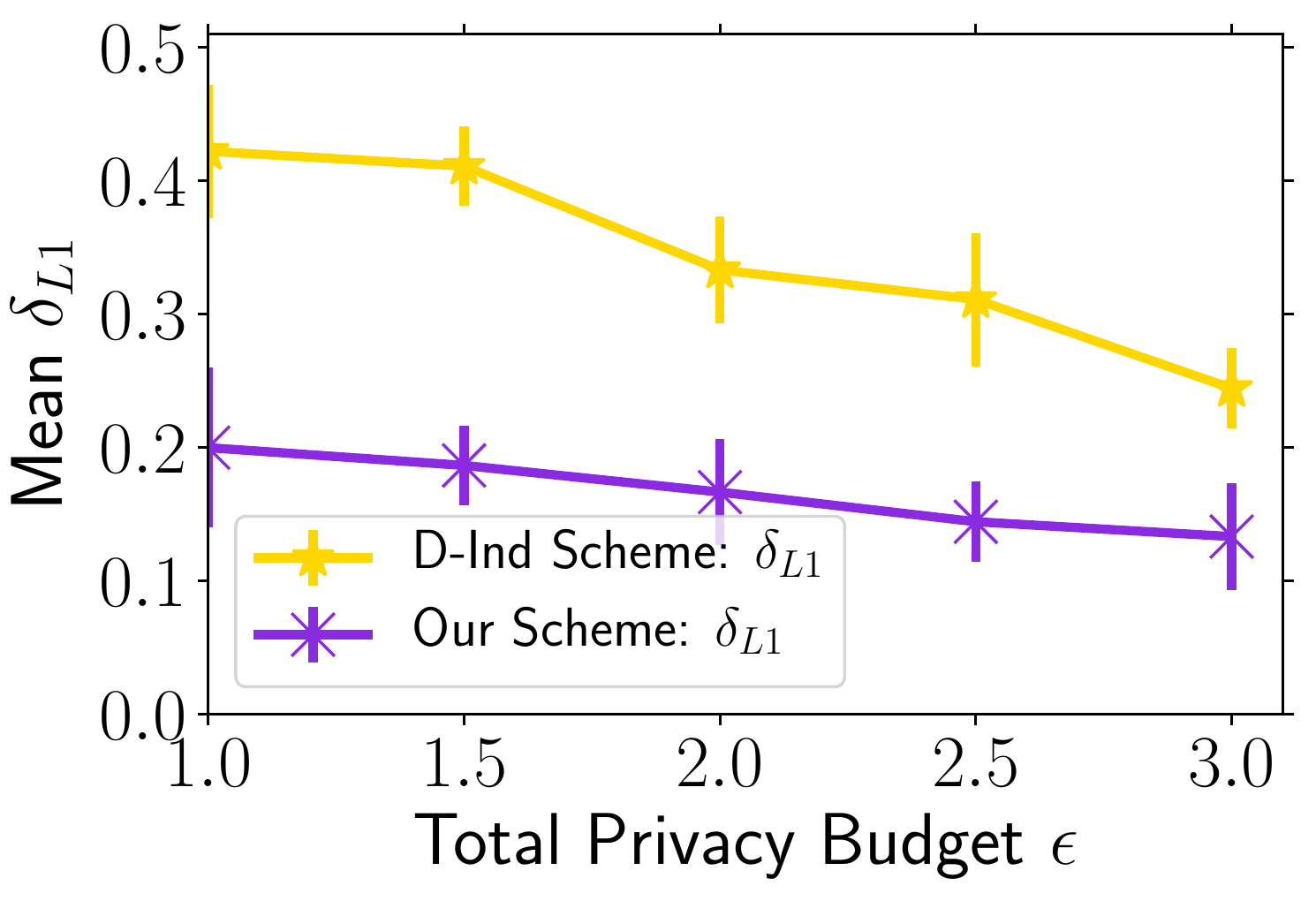}
        \caption{Sachs: Parameter $\delta_{L1}$}
        \label{fig:Para:Sachs}
    \end{subfigure}
    \begin{subfigure}[b]{0.5\columnwidth}
    \includegraphics[width=\columnwidth]{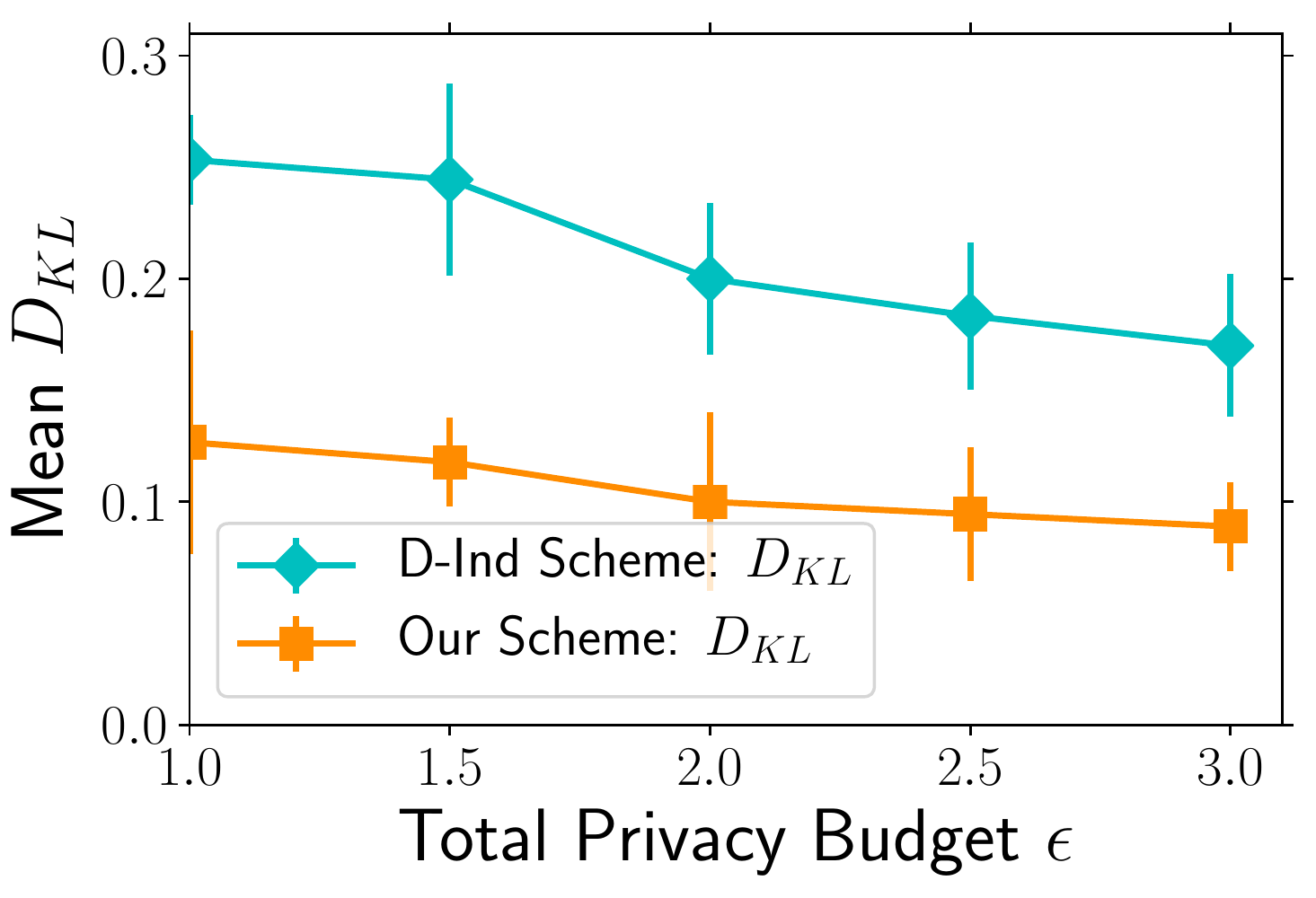}
        \caption{Sachs: Parameter $D_{KL}$}
        \label{fig:Inf:Sachs}\end{subfigure}
    \begin{subfigure}[b]{0.5\columnwidth}
    \includegraphics[width=\columnwidth]{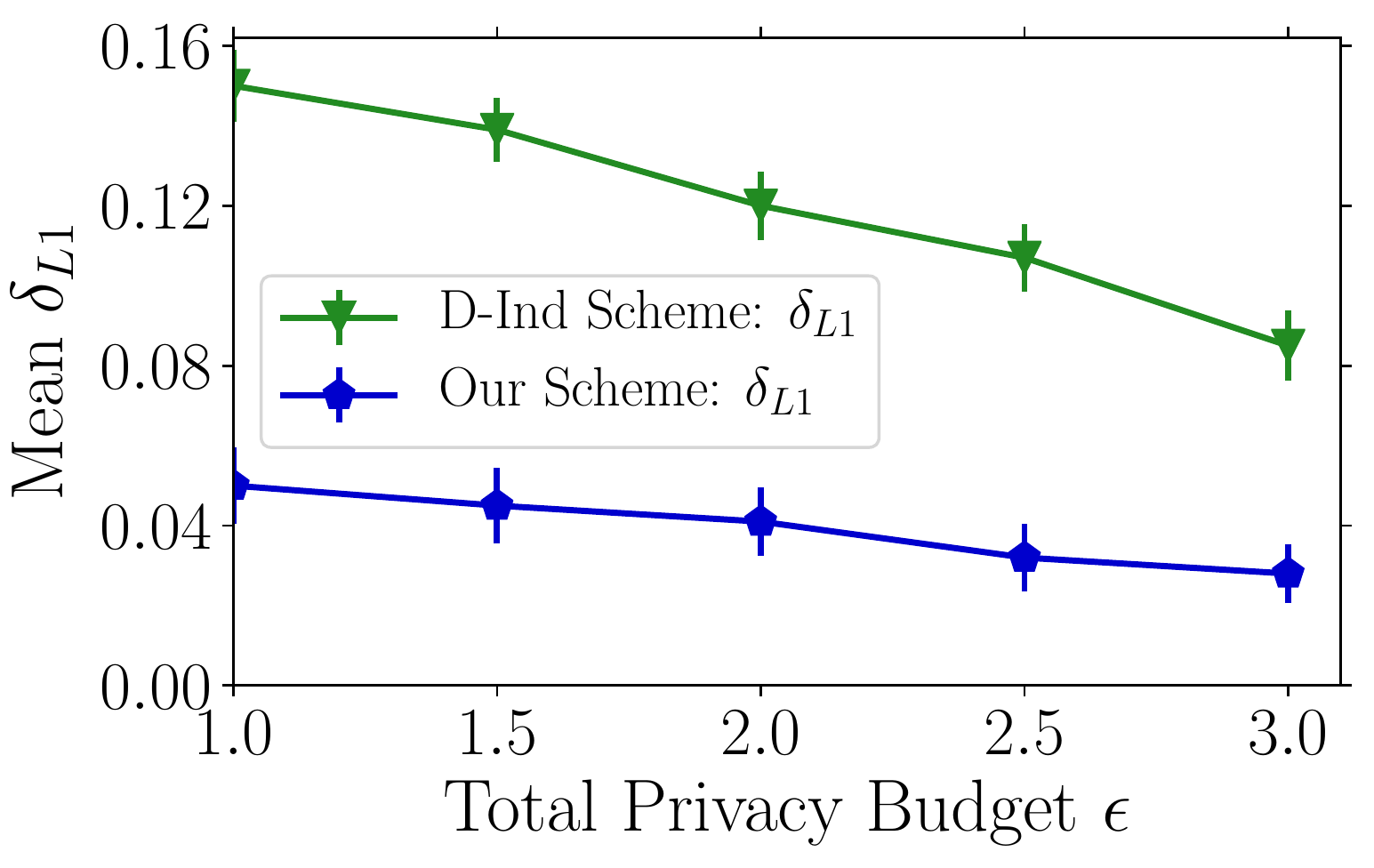}
        \caption{Sachs: $\delta_{L1}$ Inference}
        \label{fig:Para:Child}\end{subfigure}
      \begin{subfigure}[b]{0.5\columnwidth}
    \includegraphics[width=\columnwidth]{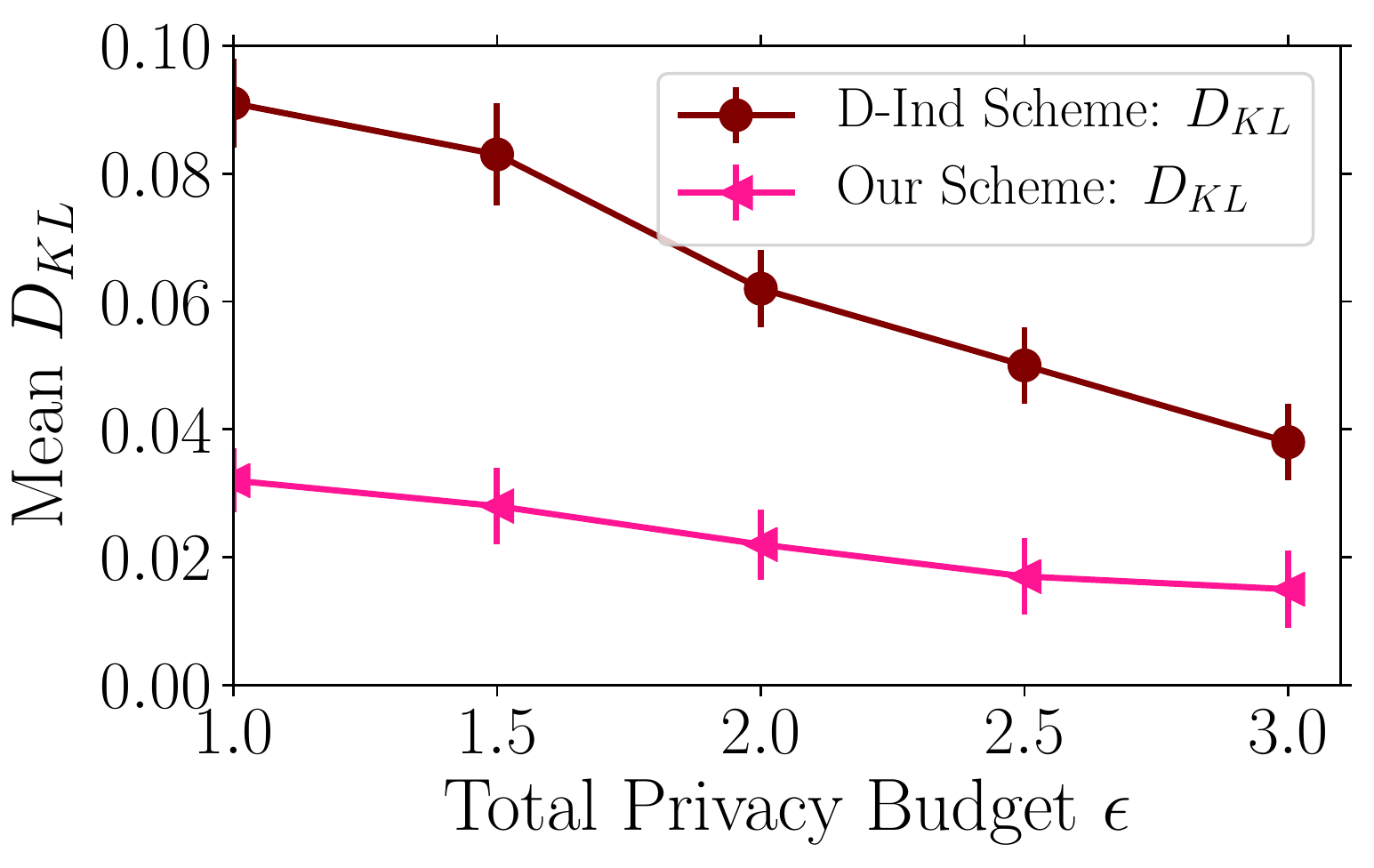}
        \caption{Sachs: Inference $D_{KL}$}
        \label{fig:Inf:Child}
    \end{subfigure}\\ 
    \begin{subfigure}[b]{0.5\columnwidth}
    \includegraphics[width=\columnwidth]{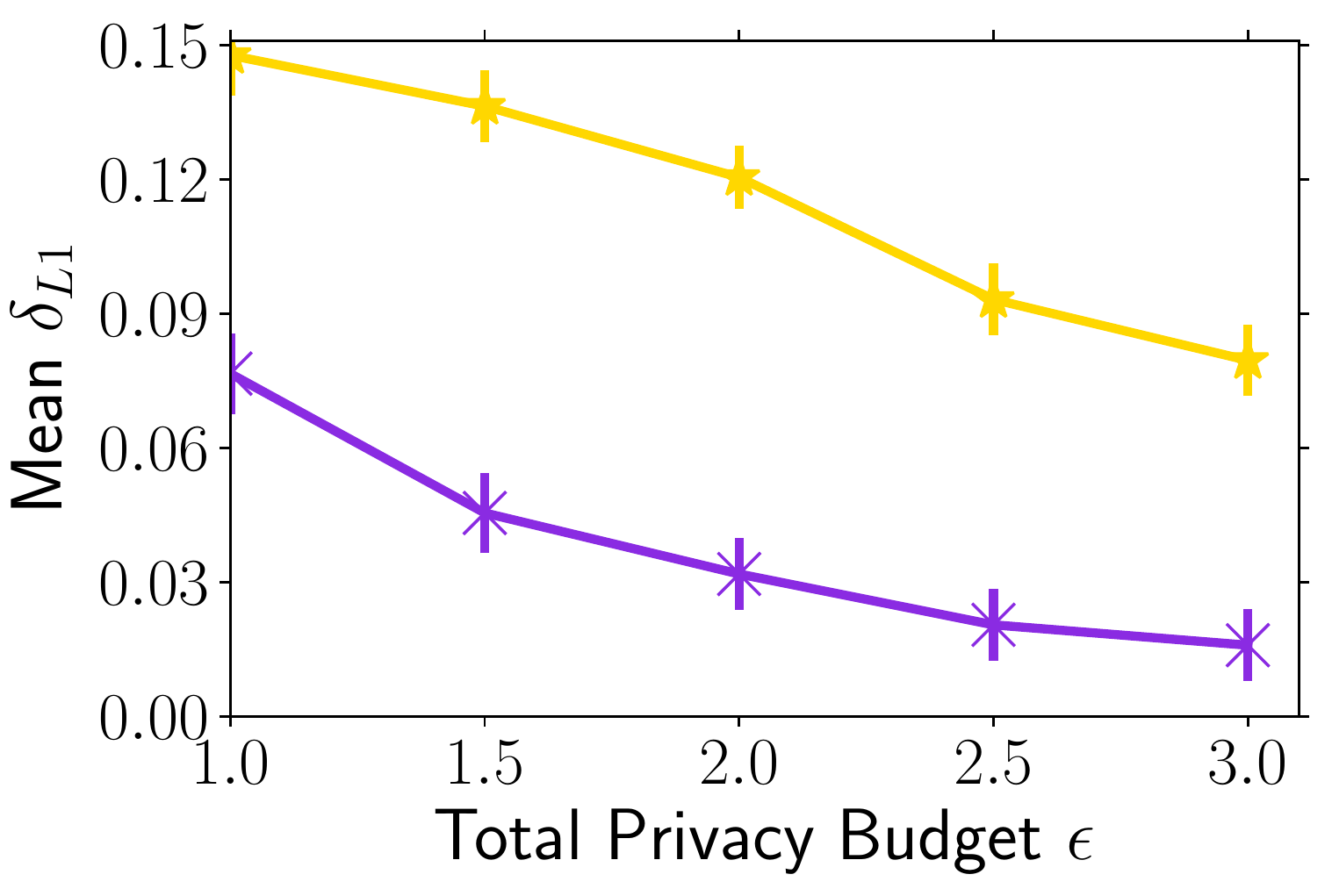}
        \caption{Child: Parameter $\delta_{L1}$}
        \label{fig:Para:Sachs}
    \end{subfigure}
    \begin{subfigure}[b]{0.5\columnwidth}
  \includegraphics[width=\columnwidth]{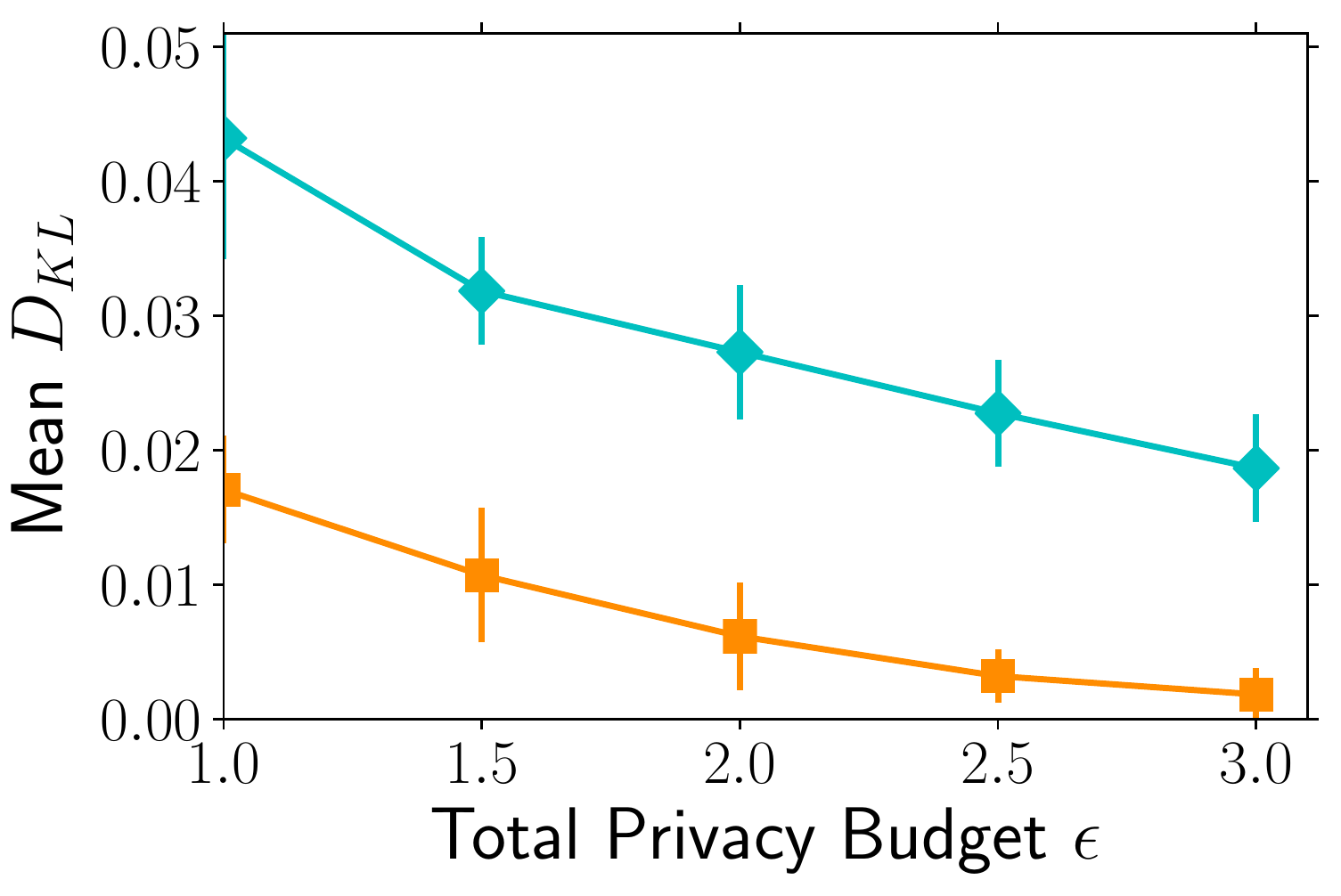}
        \caption{Child: Parameter $D_{KL}$}
        \label{fig:Inf:Sachs}\end{subfigure}
    \begin{subfigure}[b]{0.5\columnwidth}
   \includegraphics[width=\columnwidth]{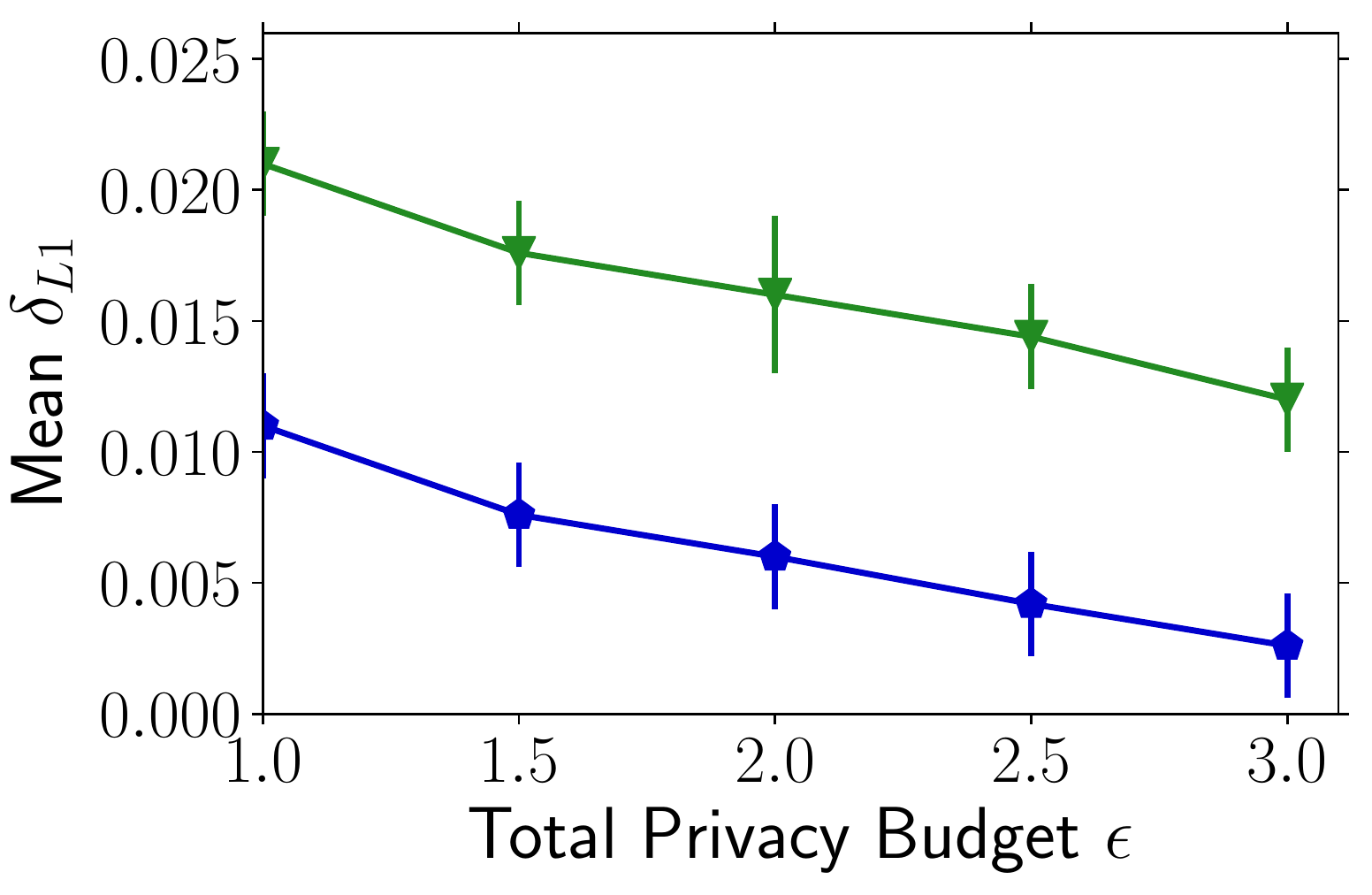}
        \caption{Child: $\delta_{L1}$ Inference}
        \label{fig:Para:Child}\end{subfigure}
      \begin{subfigure}[b]{0.5\columnwidth}
   \includegraphics[width=\columnwidth]{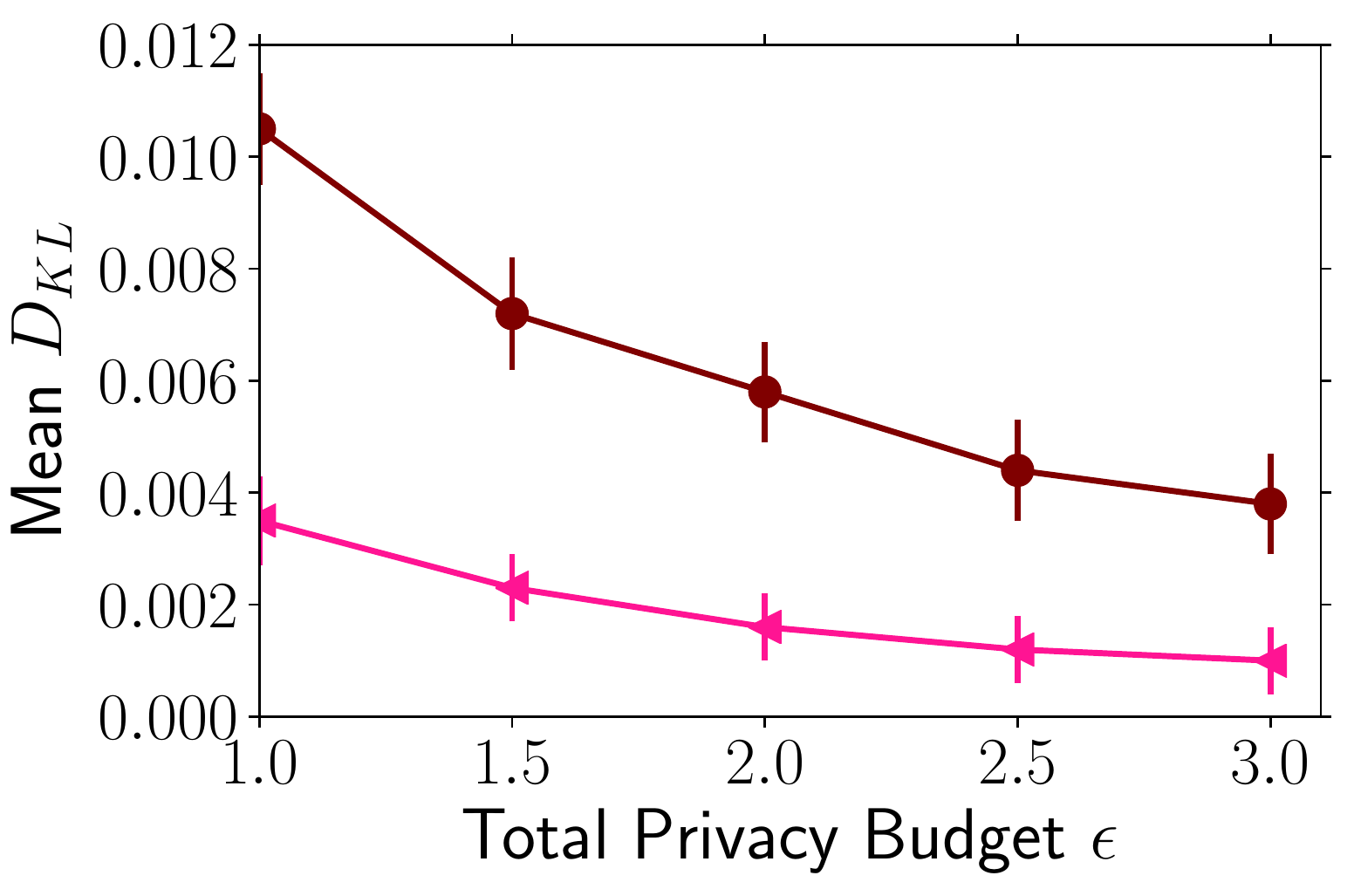}
        \caption{Child: Inference $D_{KL}$}
        \label{fig:Inf:Child}
    \end{subfigure}

    \caption{Parameter and Inference (Marginal and Conditional) Error Analysis: We observe that our scheme only requires a privacy budget of $\epsilon=1.0$ to yield the same utility that \textit{D-Ind} achieves with $\epsilon=3.0$.
    }
   \label{fig:accuracy}
    \end{figure*} As discussed in Sec. \ref{optimization}, our optimization objective minimizes a weighted sum of the parameter errors. To understand how the error propagates from the parameters to the inference queries, we present two general results bounding the error of a sum-product term of the VE algorithm, given the errors in the factors. 
\begin{thm}  \textbf{[Lower Bound]} For a DGM \scalebox{0.9}{$\mathcal{N}$}, for any sum-product term of the form \scalebox{0.9}{$\boldsymbol{\phi}_{\mathpzc{A}}=\sum_{x}\prod_{i=1}^t\phi_i, t \in \{2,\cdots,\eta\} $}  in the VE algorithm,  \begin{equation}\small\delta_{\phi_{\mathpzc{A}}} \geq \sqrt{\eta-1}\cdot\delta^{min}_{\phi_i[a,x]}(\phi^{min}_i[a,x])^{\eta-2}\label{eq:sumproduct}\end{equation} where \scalebox{0.9}{$X$} is the attribute being eliminated, \scalebox{0.9}{$\delta_{\phi}$} denotes the error in the factor \scalebox{0.9}{$\phi$}, \scalebox{0.9}{$Attr(\phi)$} is the set of attributes in \scalebox{0.9}{$\phi$}, \scalebox{0.9}{$ \mathpzc{A}=\bigcup_{\phi_i}\{Attr(\phi_i)\}/ X$},  \scalebox{0.9}{$x \in dom(X)$},  \scalebox{0.9}{$ a \in$} \scalebox{0.9}{$dom(\mathpzc{A})$}, \scalebox{0.9}{$\phi[a,x]$}  denotes that \scalebox{0.9}{$Value(Attr(\phi))\in \{a\} \wedge$}\scalebox{0.9}{$X=x$}, \scalebox{0.9}{$ \delta^{min}_{\phi_i[a,x]}=min_{i,a,x}\{\delta_{\phi_i[a,x]}\}$}, \scalebox{0.9}{$\phi^{min}_i[a,x]=min_{i,a,x}\{\phi_i[a,x]\}$} and \scalebox{0.9}{$ \eta=\max_{X_i}\{\text{in-degree}(X_i)+$} \scalebox{0.9}{$ \text{out-degree}(X_i)\}+1$}. \label{thm:lowerbound} \end{thm}
\vspace{0.1cm}
\begin{thm} \textbf{[Upper Bound]} For a DGM \scalebox{0.9}{$\mathcal{N}$}, for any sum-product term of the form \scalebox{0.9}{$ \boldsymbol{\phi}_\mathpzc{A}=\sum_x\prod_{i=1}^t\phi_i$, $ t \in \{2,\cdots,n\}$} in the  VE algorithm with the optimal elimination order, \begin{equation}\small\delta_{\phi_\mathpzc{A}} \leq 2\cdot \eta\cdot d^{\kappa}\delta^{max}_{\phi_i[a,x]}\vspace{-0.1cm} \end{equation}  where \scalebox{0.9}{$X$} is the attribute being eliminated, \scalebox{0.9}{$\delta_{\phi}$} denotes the error in the factor \scalebox{0.9}{$\phi$}, \scalebox{0.9}{$\kappa$} is the treewidth of \scalebox{0.9}{$\mathcal{G}$}, $d$ is the maximum  domain size of an attribute, \scalebox{0.9}{$Attr(\phi)$} is the set of attributes in \scalebox{0.9}{$\phi$},  \scalebox{0.9}{$\small \mathpzc{A}=\bigcup_{i}^t \{Attr(\phi_i)\}/X$}, \scalebox{0.9}{$a \thinspace\thinspace \in \thinspace\thinspace dom(\mathpzc{A}), \thinspace\thinspace\thinspace\thinspace x  
\in\thinspace\thinspace dom(X),\thinspace \thinspace\thinspace \thinspace$}\scalebox{0.9}{$ \thinspace \thinspace\thinspace\thinspace \thinspace\thinspace\phi[a,x]$} denotes that \scalebox{0.9}{$Value(Attr(\phi))\in \{a\} \wedge X=x$},  \scalebox{0.9}{$\delta^{max}_{\phi_i[a,x]}=$}
\scalebox{0.9}{$\max_{i,a,x}\{\delta_{\phi_i[a,x]}\}$ } and \scalebox{0.9}{$ \eta=\underset{X_i}{\max}\{\text{in-degree}(X_i) + \text{out-degree}(X_i)\}+1$}. \label{thm:upperbound} \end{thm}

For proving the lower bound, we introduce a specific instance of the DGM based on Lemma \ref{lemma:factor} (Appx. \ref{app:DGM}). For the upper bound, with the optimal elimination order of the VE algorithm, the maximum error has an exponential dependency on the treewidth $\kappa$. For example, for the DGM in Fig. \ref{fig:DGM}, the maximum error is bounded by \scalebox{0.9}{$2\cdot \eta\cdot d^2\cdot \delta^{max}_{\phi_i[a,x]}$}. This is very intuitive as even the complexity of the VE algorithm has the same dependency on \scalebox{0.9}{$\kappa$}. The answer of a marginal inference query is the factor generated from the last sum-product term. Also, since the initial set of \scalebox{0.9}{$\phi_i$}s for the first sum-product term computation are the actual parameters of the DGM, all the errors in the subsequent intermediate factors and hence the inference query itself can be bounded by functions of parameter errors using the above theorems.

\section{Evaluation}\label{sec:evaluation}
\begin{table*}[htb]
\begin{minipage}{0.65\linewidth}
\caption{MAP Inference Error Analysis}\label{tab:MAP}
\centering
\scalebox{0.72}{
\begin{tabular}{|c  |c |c| c| c| c| c| c |c| }
\hline
\multirow{4}{*}{$\boldsymbol{\epsilon}$} &  \multicolumn{8}{c|}{$\boldsymbol{\rho}$} \\ 
 \cline{2-9}& \multicolumn{2}{c|}{\textbf{Asia}} &  \multicolumn{2}{c|}{\textbf{Sachs}} & \multicolumn{2}{c|}{\textbf{Child}} & \multicolumn{2}{c|}{\textbf{Alarm}}   \\ \cline{2-9}& \textbf{D-Ind} & \textbf{Our } & \textbf{D-Ind } & \textbf{Our} & \textbf{D-Ind} & \textbf{Our }& \textbf{D-Ind } &\textbf{ Our }\\ & \textbf{Scheme} & \textbf{Scheme } & \textbf{Scheme} & \textbf{Scheme} & \textbf{Scheme} & \textbf{Scheme}& \textbf{Scheme} &\textbf{Scheme }\\ 
\hline
\hline
1 &0.88& 1& 0.81 & 0.86 & 0.79 & 0.93 & 0.89& 0.95 \\\hline
1.5 & 0.93 & 1  &0.87 & 0.93 & 0.83 & 0.95 & 0.92 &0.98 \\ \hline
2&  1 & 1 & 0.92&0.98&0.89&0.97 &0.95&1  \\\hline 2.5  &1& 1 &0.96& 1 & 1 &1&1&1 \\\hline 3 & 1 & 1 &1& 1& 1 & 1 &1 & 1\\
\hline
\end{tabular}}\end{minipage}%
\begin{minipage}{0.35\linewidth}
\caption{Error Bound Analysis}\label{tab:bound}
\centering
\scalebox{0.8}{
\begin{tabular}{|c  |c |c| c| c| }
\hline
 &  \multicolumn{4}{c|}{$\boldsymbol{\mu}$ } \\ 
 \cline{2-5}& \multicolumn{1}{c|}{\textbf{Asia}} &  \multicolumn{1}{c|}{\textbf{Sachs}} & \multicolumn{1}{c|}{\textbf{Child}} & \multicolumn{1}{c|}{\textbf{Alarm}}   \\ 
\hline
\hline \text{D-Ind} & \multirow{2}{*}{0.008}&\multirow{2}{*}{0.065}&\multirow{2}{*}{0.0035} & \multirow{2}{*}{0.0046}    \\ \text{Scheme} & & & & \\\hline\text{Our}
&\multirow{2}{*}{0.0035}& \multirow{2}{*}{0.04}&\multirow{2}{*}{0.0014}& \multirow{2}{*}{0.0012}\\Scheme & & & &\\\hline 
\end{tabular}} \end{minipage}

\end{table*} 

We evaluate the utility of the DGM learned via our algorithm by studying the following three questions: 

(1) Does our scheme lead to low error estimation of the DGM parameters? 

(2) Does our scheme result in low error inference query responses? 

(3) How does our scheme fare against data-independent approaches?

\textbf{Evaluation Highlights:} First, focusing on the parameters of the DGM, we find that our scheme achieves low L1 error (at most $0.2$ for $\epsilon=1$) and low KL divergence (at most $0.13$ for $\epsilon=1$) across all test data sets. Second, we find that for marginal and conditional inferences, our scheme provides low  L1 error and KL divergence (both around $0.05$ at max for $\epsilon=1$) for all test data sets. Our scheme also provides high accuracy for MAP queries ($93.5\%$ accuracy for $\epsilon=1$ averaged over all test data sets). Finally, our scheme achieves strictly better utility than the data-independent baseline; our scheme only requires a privacy budget of $\epsilon=1.0$ to yield the same utility that the  data-independent baseline achieves with $\epsilon=3.0$.
\subsection{Experimental Setup}
\textbf{Data sets:} We evaluate our proposed scheme on four benchmark DGMs \cite{BN} namely \textit{Asia, Sachs, Child} and \textit{Alarm}. For all four DGMs, the evaluation is carried out on corresponding synthetic data sets \cite{D1,D2} with 10K records each.  These data sets  are standard benchmarks for evaluating DGM
inferencing and are derived from real-world use cases \cite{BN}. Due to space constraints, we present the results for only two of the data sets (Sachs and Child) here; the rest are presented in Appx.~\ref{evaluation:cntd}. The details of Sachs and Child are as follows:

\textbf{Sachs:} Number of nodes -- 11; Number of arcs -- 17; Number of parameters -- 178

\textbf{Child:} Number of nodes -- 20; Number of arcs -- 25; Number of parameters -- 230

\textbf{Baseline:} We compare our results with a data-independent baseline (denoted by \textit{D-Ind}) which corresponds to executing Proc. 1 on the entire input data set $\mathcal{D}$ with privacy budget array \scalebox{0.9}{$\mathpzc{E}=[\frac{\epsilon^B}{n},\cdots,\frac{\epsilon^B}{n}]$}. \textit{D-Ind} is in fact based on \cite{Bayes4} (see Sec. \ref{sec:related_work} for details) which is the most recent work that explicitly deals with parameter estimation for DGMs. \textit{D-Ind} is also identical (it has an additional consistency step) to an algorithm used in PrivBayes \cite{privbayes} which uses DGMs to generate high-dimensional data.
Other candidate baselines from the differential privacy literature include MWEM \cite{MWEM}, HDMM \cite{HDMM}, Ding et al's work \cite{Datacube} and PriView \cite{PriView}. However, even with binary attributes, MWEM, \cite{Datacube} and HDMM have run time \scalebox{0.9}{$O(2^n),O(2^n)$} and \scalebox{0.9}{$O(4^n)$} respectively for marginal queries in our setting. Our scheme uses a sub protocol (\scalebox{0.9}{$MutualConsistency()$}) from PriView. However, 
PriView works best for answering all $n\choose k$ $k$-way marginals (for $k=2 \mbox
{ or } 3$). In contrast, our setting computes only $n$ marginals of varying size (often greater than 3). Hence PriView gives lower accuracy than our baseline as the privacy budget is wasted in computing irrelevant marginals. For example, we have empirically verified that for dataset Sachs, the mean error of the parameters for our baseline is an order smaller than that of PriView.

\textbf{Metrics:} For conditional and marginal inference queries we compute the following two metrics: L1-error, \scalebox{0.9}{$\delta_{L1}=\sum_{x,y}|P[x|y]-\tilde{P}[x|y]|$}and KL divergence, \scalebox{0.9}{$D_{KL}$ = $\sum_{x,y}\tilde{P}[x|y]\ln\Big(\frac{\tilde{P}[x|y]}{P[x|y]}\Big)$} where  \scalebox{0.9}{$P[x|y]$} is either a true CPD of the DGM (parameter) or a true marginal/conditional inference query response  and \scalebox{0.9}{$\tilde{P}[x|y]$} is the corresponding noisy estimate obtained from our scheme. For answering MAP inferences, we compute $\rho=\frac{\text{\# Correct answers}}{\text{\#Total runs}}$. 

\textbf{Setup:} We evaluate each data set on 20 random inference queries (10 marginal inference, 10 conditional inference) and report mean error over 10 runs.  For MAP queries, we run 20 random queries and report the mean result over 10 runs. The queries are of the form \scalebox{0.9}{$P[X|Y]$} where attribute subsets \scalebox{0.9}{$X$} and \scalebox{0.9}{$Y$} are varied from being singletons up to the full attribute set. We compare our results with a standard data-independent baseline (denoted by \textit{D-Ind}) \cite{Bayes4,privbayes} which corresponds to executing Procedure 1 on the entire input data set $\mathcal{D}$ and the privacy budget array \scalebox{0.9}{$\mathpzc{E}=[\frac{\epsilon^B}{n},\cdots,\frac{\epsilon^B}{n}]$}. 
All the experiments have been implemented in Python and we set  \scalebox{0.9}{$e^{I}=0.1\cdot e^B, \beta=0.1$}. 
\subsection{Experimental Results}
Fig. \ref{fig:accuracy} shows the mean \scalebox{0.9}{$\delta_{L1}$} and \scalebox{0.9}{$D_{KL}$} for the noisy parameters, and the marginal and conditional inferences for the data sets Sachs and Child. The main observation is that our scheme achieves strictly lower error than that of \textit{D-Ind}. Specifically, our scheme only requires a privacy budget of $\epsilon=1.0$ to yield the same utility that \textit{D-Ind} achieves with $\epsilon=3.0$. In most practical scenarios, the value of $\epsilon$ typically does not exceed $3$ ~\cite{epsilon}. In Table \ref{tab:MAP}, we present our experimental results for MAP queries. We see that our scheme achieves higher accuracy than \textit{D-Ind}. For example, our scheme provides an accuracy of at least $86\%$ while \textit{D-Ind} achieves $81\%$ accuracy for $\epsilon=1$. Finally, given a marginal inference query \scalebox{0.9}{$\mathcal{Q}$}, we compute the scale normalized error in \scalebox{0.9}{$\mathcal{Q}$} as $\mu=\frac{\delta_{L1}[\mathcal{Q}]-LB}{UB-LB}$ where \scalebox{0.9}{$UB$} and \scalebox{0.9}{$LB$} are the upper and lower bound computed using Thm. \ref{thm:upperbound} and Thm. \ref{thm:lowerbound}\footnote{$UB$ and $LB$ are computed separately for each run of the experiment from their respective empirical parameter errors.} respectively. Clearly, the lower the value of $\mu$, the closer it is to the lower bound and vice versa. We report the mean value of $\mu$ for 20 random inference queries (marginal and conditional) for $\epsilon=1$ in Table \ref{tab:bound}. We observe that the errors are closer to their respective lower bounds. This is more prominent for the errors obtained from our data-dependent scheme than those of \textit{D-Ind}.

Thus, we conclude that the non-uniform budget  \begin{wrapfigure}{r}{0.47\columnwidth}\hspace{-0.2cm}\centering\includegraphics[width=0.47\columnwidth]{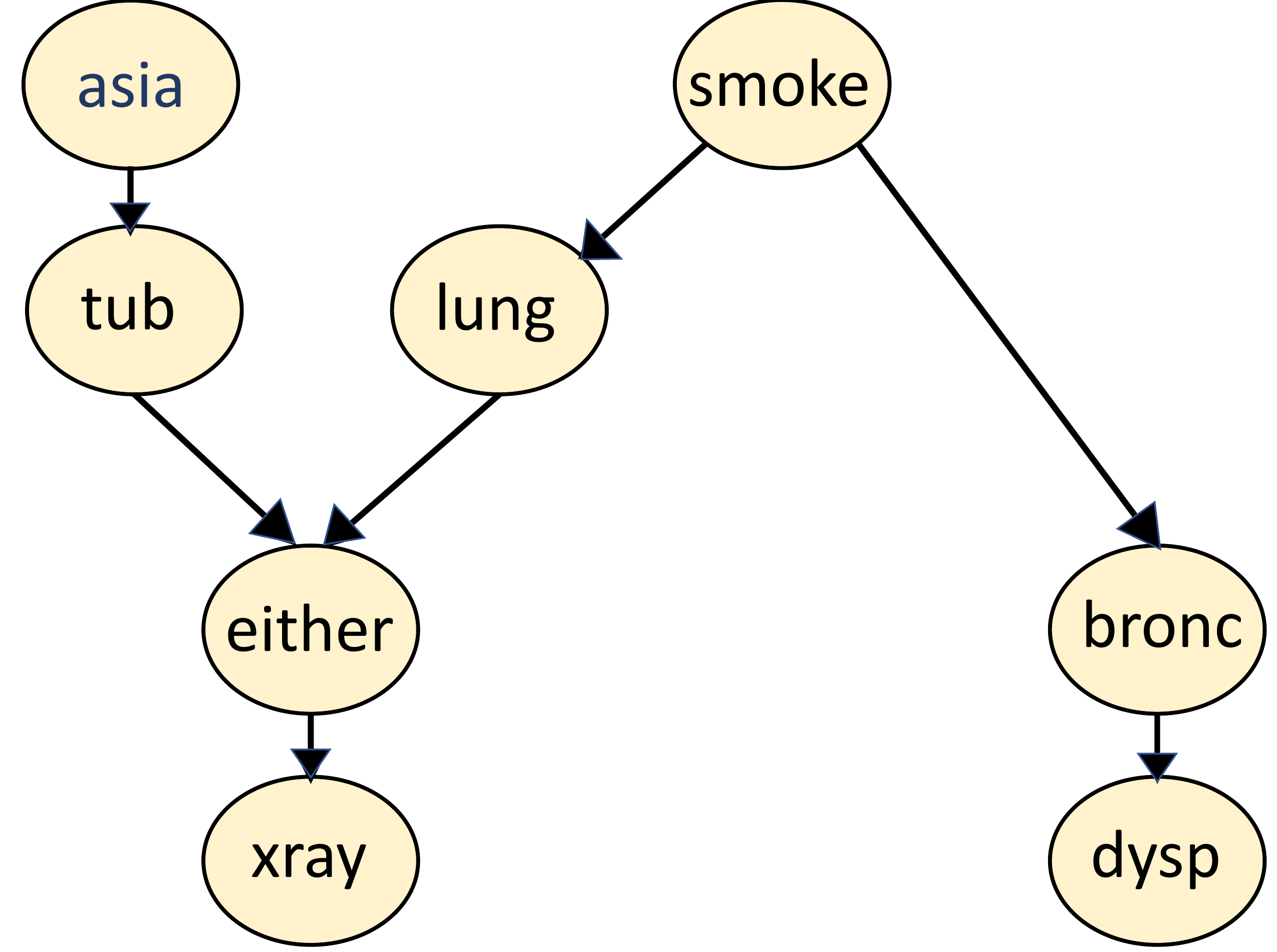}\caption{Graph Structure of DGM Asia}\vspace{-0.3cm}
\label{fig:Asia}
\end{wrapfigure}  allocation in our data-dependent scheme gives better utility than uniform budget allocation. For example, for DGM Asia (Fig. ~\ref{fig:Asia}), our scheme allocates the privacy budget in the following order: "asia"$>$"smoke"$>$"bronc" $>$"lung"$>$"tub"$>$"xray"$>$ "either"$>$"dysp". As expected, nodes with greater height and out-degree are assigned higher budget than leaf nodes.

\section{Related Work}\label{sec:related_work}
Here we briefly discuss the related literature (see Appx.~\ref{app:relatedWork} for a detailed review). There has been a fair amount of work in differentially private Bayesian inferencing  \cite{Bayes1,Bayes2,Bayes3,Bayes4,Bayes5,Bayes6, Bayes7,Bayes9,Data_PAC,Bayes10,Bayes11,Bayes12,Bayes13,Bayes14,Bayes15, MRF}.  All of the above works have different setting/goal from our paper. Specifically, in \cite{Bayes4} the authors propose algorithms for private Bayesian inference on graphical models. However, their proposed solution does not add data-dependent noise. In fact, their proposed algorithms (Alg. 1 and Alg. 2 in \cite{Bayes4}) are essentially the same in spirit as our baseline solution, \textit{D-Ind}. Moreover, some proposals from \cite{Bayes4} can be combined with \textit{D-Ind}. For example, to ensure mutual consistency, \cite{Bayes4} adds Laplace noise in the Fourier domain  while \textit{D-Ind} uses techniques of \cite{Consistency}. \textit{D-Ind} is also identical (it has an additional consistency step) to an algorithm used in \cite{privbayes} which uses DGMs to generate high-dimensional data. Data-dependent noise addition is a popular technique in differential privacy 
\cite{hist,hist2,AHP,DPcube, MWEM,DAWA,DataAshwin}.

\section{Conclusion}
In this paper, we have proposed an algorithm for differentially private learning of the parameters of a DGM with a publicly known graph structure over fully observed data. The noise added is customized to the private input data set as well as the public graph structure of the DGM. To the best of our knowledge, we propose the first explicit data-dependent privacy budget allocation mechanism in the context of DGMs. Our solution achieves strictly higher utility than that of a standard data-independent approach; our solution requires roughly $3\times$ smaller privacy budget to achieve the same or higher utility.
\bibliographystyle{icml2020}
\bibliography{references.bib}
\pagebreak
\twocolumn[\icmltitle{Data-Dependent Differentially Private Parameter Learning for Directed Graphical Models -- Supplementary Material }]

\section{Appendix}
\subsection{Background Cntd.}
\subsubsection{Directed Graphical Models Cntd.}\label{app:DGM}

\subsubsection*{Learning Cntd.}\label{learning:cntd}
As mentioned before in Section 2, for a fully observed DGM, the parameters (CPDs) are computed via  maximum likelihood estimation (MLE). Let the data set be $\mathcal{D}=\{x^{1},x^{2},\cdots,x^{m}\}$ with $m$ i.i.d records with attribute set $\langle X_1,X_2, \cdots, X_n \rangle$. 
 The likelihood is then given by \begin{gather}\mathcal{L}(\Theta,\mathcal{D})=\prod_{i=1}^n\prod_{j=1}^m \Theta_{x_i^{j}|x_{pa_i}^{j}}\end{gather} where $\Theta_{x_i^{j}|x_{pa_i}^{j}}$ represents the probability that $X_i=x_i^{j}$ given $X_{pa_i}=x_{pa_i}^{j}$. 
 
Taking logs and rearranging, this reduces to
\begin{gather}
log L(\Theta,\mathcal{D})=\sum_{i=1}^n\sum_{x_{pa_i}}\sum_{x_i}
 C[x_i,x_{pa_i}]\cdot \Theta[x_i|x_{pa_i}]
\end{gather} where $C[x_i,x_{pa_i}]$ denotes the number of records in data set $\mathcal{D}$ such that $X_i=x_i,X_{pa_i}=x_{pa_i}$. Thus, maximization of the (log) likelihood function decomposes into separate maximizations for the local conditional distributions which results in the closed form solution 
 \begin{gather}\Theta_{x_i|x_{pa_i}}=\frac{C[x_i,x_{pa_i}]}{C[x_{pa_i}]} \label{learning}\end{gather}
\textbf{Variable Elimination Algorithm (VE)}: The complete VE algorithm is given by Algorithm 3. The basic idea of the variable elimination algorithm is that we "\textit{eliminate}" one variable at a time following a predefined order $\prec$ over the nodes of the graph. Let $\Phi$ denote a set of probability factors which is initialized as the set of all CPDs of the DGM and $Z$ denote the variable to be eliminated.  For the elimination step, firstly all the probability factors involving the variable to be eliminated, $Z$ are removed from $\Phi$ and multiplied together to generate a new product factor. Next, the variable $Z$ is summed out from this combined factor generating a new factor that is entered into $\Phi$. Thus 
the VE algorithm essentially involves repeated computation of a sum-product task of the form
 \begin{gather}\boldsymbol{\phi}=\sum_Z\prod_{\phi \in \Phi} \phi \label{VE}
 \end{gather} The complexity of the VE algorithm is defined by the size of the largest factor.
Here we state two lemmas regarding the intermediate factors $\boldsymbol{\phi}$ which will be used in Section \ref{error:cntd}. \begin{lem} Every intermediate factor ($\boldsymbol{\phi}$ in \eqref{VE}) generated as a result of executing the VE algorithm on a DGM  $\mathcal{N}$ correspond to a valid conditional probability of some DGM (not necessarily the same DGM, $\mathcal{N}$). \label{lemma:factor} \cite{PGMbook}\end{lem}
\begin{lem}The size of the largest intermediary factor generated as a result of running of the VE algorithm on a DGM  is at least equal to the treewidth of the graph \cite{PGMbook}. \label{lemma:treewidth} \end{lem} \begin{cor}The complexity of the VE algorithm with the optimal order of elimination depends on the treewidth of the graph. \label{cor}\end{cor}
\begin{algorithm}[t]
\caption{Sum Product Variable Elimination Algorithm}\label{alg:VE}
\begin{algorithmic}[1]
\Statex \textbf{Notations} : $\Phi$ -  Set of factors
\Statex \hspace{1.67cm} $\mathbf{X}$ - Set of variables to be eliminated
\Statex \hspace{1.67cm} $\prec$ - Ordering on $\mathbf{X}$
\Statex \hspace{1.67cm} $X$ - Variable to be eliminated
\Statex \hspace{1.67cm} $Attr(\phi)$ - Attribute set of factor $\phi$
\STATEx \textbf{Procedure} Sum-Product-VE($\Phi,\mathbf{X},\prec$)
\STATE \hspace{0.5cm} Let $X_1,\cdots,X_k$ be an ordering  of $\mathbf{X}$ such that
$X_i \prec X_j$ iff $i<j$
\STATE \hspace{0.5cm} \textbf{for} $i=1,\cdots,k$
\STATE \hspace{0.9cm} $\Phi \leftarrow $ Sum-Product-Eliminate-Var($\Phi,Z_i$)
\STATE \hspace{0.9cm} $\phi^*\leftarrow \prod_{\phi \in \Phi}\phi $\STATE \hspace{0.5cm} \textbf{return} $\phi^*$
\STATEx \textbf{Procedure} Sum-Product-Eliminate-Var($\Phi,X$)
\STATE \hspace{0.5cm} $\Phi' \leftarrow \{\phi \in \Phi : Z \in Attr(\phi)\}$
\STATE \hspace{0.5cm} $\Phi'' \leftarrow \Phi - \Phi'$
\STATE \hspace{0.5cm} $\psi \leftarrow \prod_{\phi \in \Phi'}\phi$
\STATE \hspace{0.5cm} $\boldsymbol{\phi} \leftarrow \sum_{Z}\psi$
\STATE \hspace{0.5cm} \textbf{return} $\Phi''\cup \{\boldsymbol{\phi}\}$
\end{algorithmic}
\end{algorithm}
\subsection{Data-Dependent Differentially Private Parameter Learning for DGMs Cntd.}

\subsubsection{Consistency between noisy marginal tables}\label{consistency}
The objective of this step is to input the set of noisy marginal tables $\tilde{M}_i$ and compute perturbed versions of these tables that are mutually
consistent (Defn. \ref{def:mutualconsistency}.
The following procedure has been reproduced from \cite{Consistency,PriView} with a few adjustments.

\textbf{Mutual Consistency on a Set of Attributes:} 

Assume a set of tables $\{\tilde{M}_i,\cdots,\tilde{M}_j\}$ and let $A = Attr(\tilde{M}_i) \cap \cdots \cap Attr(\tilde{M}_j)$. Mutual consistency, i.e., $\tilde{M}_i[A] \equiv \cdots \equiv \tilde{M}_j[A]$ is achieved as follows: 

(1) First
compute the best approximation for the marginal table $\tilde{M}_A$ for the attribute set $A$ as follows \begin{gather}\tilde{M}_A[A']= \frac{1}{\sum_{t=1}^j{\epsilon_t}}\sum_{t=i}^j \epsilon_t\cdot\tilde{M}_t[A'], A' \in A\end{gather}
\\ (2) Update all $\tilde{M}_t$s to be consistent with
$\tilde{M}_A$.  Any counting query $c$ is now answered as
\begin{gather}
\tilde{M}_t(c) = \tilde{M}_t(c) + \frac{|dom(A)|}{|dom(Attr\big(\tilde{M})|}(\tilde{M}_A(a) - \tilde{M}_t (a)\big)
\end{gather}
where $a$ is the query $c$ restricted to attributes in $A$ and $\tilde{M}_t(c)$ is the response of $c$ on $\tilde{M}_t$ .

\textbf{
Overall Consistency:} 

(1) Take all sets of attributes that
are the result of the intersection of some subset of $\bigcup_{i=k+1}^d\{X_i \cup X_{pa_i}\}$; these
sets form a partial order under the subset relation.

(2) Obtain a topological sort of these sets, starting from
the empty set.

(3) For each set $A$, one finds all tables that
include $A$, and ensures that these tables are consistent on $A$.

\subsubsection{Privacy Analysis }\label{Privacy proof}
\begin{customthm}{3.1}\textit{The proposed algorithm (Algorithm \ref{algo:main}) for learning the parameters of a fully observed directed graphical model is $\epsilon^B$-differentially private.}\end{customthm}
\begin{proof}The sensitivity of counting queries is 1. Hence, the computation of the noisy tables $\tilde{T}_i$ (Proc. 1, Line 2-3) is a straightforward application of Laplace mechanism (Sec. \ref{sec:background}). This together with Lemma \ref{sampling} makes the computation of $\tilde{T}_i$, $\epsilon^I$-DP. Now the subsequent computation of the optimal privacy budget allocation $\mathpzc{E}^*$ is a post-processing operation on $\tilde{T}_i$ and hence by Thm. \ref{theorem:post} is still $\epsilon^I$-DP. The final parameter computation is clearly ($\epsilon^B$-$\epsilon^I$)-DP. Thus by the theorem of sequential composition (Thm.~\ref{theorem:seq}), Algorithm \ref{algo:main} is $\epsilon^B$-DP.\end{proof}

\subsection{Error Bound Analysis Cntd.}\label{error:cntd}
In this section, we present the proofs of Thm. \ref{thm:lowerbound} and Thm. \ref{thm:upperbound}. \\
\\\textbf{Preliminaries and Notations:}\\
For the  proofs, we use the following notations.
Let $X$ be the attribute that is being eliminated and let $\mathpzc{A}=\bigcup_{\phi_i}Attr(\phi_i)\textbackslash X$ where $Attr(\phi)$ denotes the set of attributes in $\phi$.
For some $a \in dom(\mathpzc{A})$, from the variable elimination algorithm (Sec. \ref{app:DGM}) for a sum-product term (Eq. \eqref{VE}) we have \begin{gather}\boldsymbol{\phi}_\mathpzc{A}[a]=\sum_{x}\prod_{i=1}^t\phi_i[x,a]\label{eq:sumproduct}\end{gather} Let us assume that factor $\phi[a,x]$  denotes that $Value(Attr(\phi))\in \{a\}$ and $X=x$. Recall that after computing a sum-product task (given by Eq. \eqref{eq:sumproduct}), for the variable elimination algorithm (Appx. Algorithm \ref{alg:VE}),  we will be left with a factor term over the attribute set $\mathpzc{A}$. For example, if the elimination order for the variable elimination algorithm on our example DGM (Figure \ref{fig:DGM}) is given by $\prec=\{A,B,C,D,E,F\}$ and the attributes are binary valued, then the first sum-product task will be of the following form $\mathpzc{A}=\{B,C\}, dom(\mathpzc{A})=\{(0,0),(0,1),(1,0),(1,1)\}$ and the RHS  $\phi_is$ in this case happen to be the true parameters of the DGM, \begin{gather*}
\boldsymbol{\phi}_{B,C}[0,0]=\Theta[A=0]\cdot \Theta[C=0|A=0,B=0]\thinspace+\\\hspace{3cm}\Theta[A=1]\cdot \Theta[C=0|A=1,B=0]\\\boldsymbol{\phi}_{B,C}[0,1]=\Theta[A=0]\cdot \Theta[C=1|A=0,B=0]\thinspace+\\\hspace{3cm}\Theta[A=1]\cdot \Theta[C=1|A=1,B=0]\\\boldsymbol{\phi}_{B,C}[1,0]=\Theta[A=0]\cdot \Theta[C=1|A=0,B=0]\thinspace+\\\hspace{3cm}\Theta[A=1]\cdot \Theta[C=1|A=1,B=0]\\\boldsymbol{\phi}_{B,C}[1,1]=\Theta[A=0]\cdot \Theta[C=1|A=0,B=1]\thinspace + \\\hspace{3cm}\Theta[A=1]\cdot \Theta[C=1|A=1,B=1] \\\boldsymbol{\phi}_{B,C}=[\boldsymbol{\phi}_{B,C}[0,0],\boldsymbol{\phi}_{B,C}[0,1],\boldsymbol{\phi}_{B,C}[1,0],\boldsymbol{\phi}_{B,C}[1,1]]\end{gather*}

\subsubsection{Lower Bound}
\begin{customthm}{4.1}\textit{For a DGM \scalebox{0.9}{$\mathcal{N}$}, for any sum-product term of the form \scalebox{0.9}{$\boldsymbol{\phi}_{\mathpzc{A}}=\sum_{x}\prod_{i=1}^t\phi_i, t \in \{2,\cdots,\eta\} $}  in the VE algorithm,  \begin{equation}\small\delta_{\phi_{\mathpzc{A}}} \geq \sqrt{\eta-1}\cdot\delta^{min}_{\phi_i[a,x]}(\phi^{min}_i[a,x])^{\eta-2}\label{eq:sumproduct}\end{equation} where \scalebox{0.9}{$X$} is the attribute being eliminated, \scalebox{0.9}{$\delta_{\phi}$} denotes the error in factor \scalebox{0.9}{$\phi$}, \scalebox{0.9}{$Attr(\phi)$} is the set of attributes in \scalebox{0.9}{$\phi$}, \scalebox{0.9}{$ \mathpzc{A}=\bigcup_{\phi_i}\{Attr(\phi_i)\}/ X$},  \scalebox{0.9}{$x \in dom(X)$},  \scalebox{0.9}{$ a \in$} \scalebox{0.9}{$dom(\mathpzc{A})$}, \scalebox{0.9}{$\phi[a,x]$}  denotes that \scalebox{0.9}{$Value(Attr(\phi))\in \{a\} \wedge$}\scalebox{0.9}{$X=x$}, \scalebox{0.9}{$ \delta^{min}_{\phi_i[a,x]}=min_{i,a,x}\{\delta_{\phi_i[a,x]}\}$}, \scalebox{0.9}{$\phi^{min}_i[a,x]=min_{i,a,x}\{\phi_i[a,x]\}$} and \scalebox{0.9}{$ \eta=\max_{X_i}\{\text{in-degree}(X_i)+$} \scalebox{0.9}{$ \text{out-degree}(X_i)\}+1$}.} 
\end{customthm}
\begin{proof}
\textbf{Proof Structure:}

The proof is structured as follows. First, we compute the error for a single term $\boldsymbol{\phi}_{\mathpzc{A}}[a], a  \in dom(\mathpzc{A})$ (Eq. \eqref{lower1},\eqref{lower2},\eqref{lower3}). Next we compute the total error $\delta_{\boldsymbol{\phi}_{\mathpzc{A}}}$ by summing over $\forall a \in dom(\mathpzc{A})$. This is done by dividing the summands into two types of terms (a) $\Upsilon_{\phi_1[a,x]}$ (b) $\delta_{\prod_{i=1}^t\phi_i[a,x]}\prod_{i=1}^t\phi_i[a,x]$ (Eq. \eqref{eq:Upsilon},\eqref{eq:twoterms}). We prove that the summation of first type of terms ($\Upsilon_{\phi_1[x]}$) can be lower bounded by $0$ non-trivially. Then we compute a lower bound on the terms of the form $\delta_{\prod_{i=2}^t\phi_i[a,x]}\prod_{i=2}^t\phi_i[a,x]$ (Eq. \eqref{eq:secondterm}) which gives our final answer (Eq. \eqref{eq:lowerbound}). 

\textbf{Step 1: Computing error in a single term $\boldsymbol{\phi}_{\mathpzc{A}}[a]$, $\delta_{\boldsymbol{\phi}_{\mathpzc{A}}[a]}$}\\The error in  $\boldsymbol{\phi}_\mathpzc{A}[a]$, due to noise injection is given by ,
\begin{gather*} \delta_{\boldsymbol{\phi}_{\mathpzc{A}}[a]}=\Bigg|\sum_{x}\prod_{i=1}^t\phi_i[x,a]-\sum_{x}\prod_{i=1}^t\tilde{\phi}_i[a,x]\Bigg| \\=\Bigg| \sum_x\Big(\phi_1[x,a]\prod_{i=2}^t\phi_i[x,a]-\tilde{\phi}_1[x,a]\prod_{i=2}^t\tilde{\phi}_i[x,a]\Big)\Bigg|\\=\scalebox{0.86}{$\Bigg| \sum_x\Big(\phi_1[x,a]\prod_{i=2}^t\phi_i[x,a]-\tilde{\phi}_1[x,a]\prod_{i=2}^t(\phi_i[x,a]\pm\delta_{{\phi}_i[x,a]})\Big)\Bigg|$}
\numberthis \label{lower1}\end{gather*}
Using the rule of standard error propagation,  we have
\begin{gather}\delta_{\prod_{i=2}^t{\phi_i[x,a]}} = \prod_{i=2}^t\tilde{\phi}_i[x,a] \sqrt{\sum_{i=2}^t\Big(\frac{\delta_{\phi_i[x,a]}}{\phi_i[x,a]}\Big)^2}\label{lower2} \end{gather}
Thus from the above equation (Eq. \eqref{lower2}) we can rewrite Eq. \eqref{lower1} as follows,
\begin{gather*}\scalebox{0.8}{$= \Bigg| \sum_x\Big(\phi_1[x,a]\prod_{i=2}^t\phi_i[x,a]-\tilde{\phi_1[x,a]}\prod_{i=2}^t\phi_i[x,a](1\pm\delta_{\prod_{i=2}^t\phi_i[x,a]})\Big)\Bigg|$}\\\scalebox{0.8}{$=\Bigg|\sum_x\Big((\phi_1[a,x]-\tilde{\phi_1[a,x]})\prod_{i=2}^t\phi_i[a,x]\pm\delta_{\prod_{i=1}^t\phi_i[a,x]}\prod_{i=1}^t\phi_i[a,x]\Big)\Bigg|$}\numberthis \label{lower3}\end{gather*}
\textbf{Step 2: Compute total error $\delta_{\boldsymbol{\phi}_\mathpzc{A}}$}\\
Now, total error in $\phi_\mathpzc{A}$ is 
\begin{gather}\delta_{\phi_\mathpzc{A}}=\sum_a\delta_{\boldsymbol{\phi}_\mathpzc{A}[a]}\label{total}\end{gather}

Collecting all the product terms from the above equation \eqref{total} with $\phi_1[a,x]-\tilde{\phi}_1[a,x]$  as a multiplicand, we get \begin{gather}\Upsilon_{\phi_1[a,x]}=(\phi_1[a,x]-\tilde{\phi}_1[a,x])\sum_a\prod_{i=2}^t\phi_i[a,x] \label{eq:Upsilon}\end{gather}
Thus $\delta_{\boldsymbol{\phi}_\mathpzc{A}}$ can be rewritten as \begin{gather}\delta_{\boldsymbol{\phi}_\mathpzc{A}}=\sum_{a,x}{\Upsilon}_{\phi_1[a,x]}\pm\sum_{a,x}\prod_{i=1}^t\phi_i[a,x]\delta_{\prod_{i=2}^t\phi_i[a,x]}\label{eq:twoterms}\end{gather} 

First we show that for a specific DGM we have  $\sum_{a,x}\Upsilon_{\phi_1[a,x]} =0$ as follows. Let us assume that the DGM  has $Attr(\phi_1)=X$. Thus $\phi_1[a,x]$ reduces to just $\phi_1[x]$. \begin{gather*}\Upsilon_{\phi_1[x]}=(\phi_1[x]-\tilde{\phi}_1[x])\sum_a\prod_{i=2}^t\phi_i[x]\\=(\phi_1[x]-\tilde{\phi}_1[x])\Big(\sum_{a_k}\cdots\sum_{a_1}\prod_{i=2}^t\phi_{i}[ a_1,\cdots,a_k,x]\Big)\\ [ \mathpzc{A}=\langle \mathpzc{A}_1, \cdots \mathpzc{A}_k \rangle, a_j \in dom(\mathpzc{A}_j), j \in [k]]
\\=\scalebox{0.77}{$(\phi_1[x]-\tilde{\phi}_1[x])\Big(\sum_{a_k}\cdots\sum_{a_2}\prod_{i=3}^t\phi_{i}[ a_2,\cdots,a_k,x]\sum_{a_1}\phi_{2}[a_1,\cdots,a_k,x]\Big)$} \\\hspace{0.5cm}[\text{Assuming that } \phi_2 \mbox{ is the only factor with attribute $\mathpzc{A}_1$}]\end{gather*} 
Now each factor $\phi_i$ is either a true parameter (CPD) of the DGM $\mathcal{N}$ or a CPD over some other DGM (lemma \ref{lemma:factor}). Thus, let us assume that $\phi_2$ represents a conditional of the form $P[\mathpzc{A}_1|\mathbf{A},X], \mathbf{A} = \mathpzc{A}/\mathpzc{A}_1$. Thus we have $\sum_{a_1}\phi_2[a_1,\cdots,a_k,x]=\sum_{a_1}P[\mathpzc{A}_1=a_1|\mathpzc{A}_2=a_2,\cdots,\mathpzc{A}_k=a_k,X=x]=1$.
Now repeating the above process over all $i \in \{3,\cdots,t\}  \phi_i$s , we get \begin{gather}\Upsilon_{\phi_1[x]}=\phi_1[x]-\tilde{\phi}_1[x] \label{eq:upsilon2}\end{gather}
For the ease of understanding, we illustrate the above result on our example DGM (Figure \ref{fig:DGM}). Let us assume that the order of elimination is given by $\prec=\langle A,B,C,D,E,F\rangle$.  For simplicity, again we  assume binary attributes. Let $\phi_C$ be the factor that is obtained after  eliminating $A$ and $B$. Thus the sum-product task for eliminating $C$ is given  by \begin{gather*}\boldsymbol{\phi}_{D,E}[0,0]=\phi_{C}[C=0]\cdot\Theta[D=0|C=0]\Theta[E=0|C=0] \\\hspace{1.8cm}+ \phi_{C}[C=1]\cdot\Theta[D=0|C=1]\Theta[E=0|C=1] \\\boldsymbol{\phi}_{D,E}[0,1]=\phi_{C}[C=0]\cdot\Theta[D=0|C=0]\Theta[E=1|C=0] \\\hspace{1.8cm}+ \phi_{C}[C=1]\cdot\Theta[D=0|C=1]\Theta[E=1|C=1]\\\boldsymbol{\phi}_{D,E}[1,0]=\phi_{C}[C=0]\cdot\Theta[D=1|C=0]\Theta[E=0|C=0]\\ \hspace{1.8cm}+ \phi_{C}[C=1]\cdot\Theta[D=1|C=1]\Theta[E=0|C=1]\\\boldsymbol{\phi}_{D,E}[1,1]=\phi_{C}[C=0]\cdot\Theta[D=1|C=0]\Theta[E=1|C=0] \\\hspace{1.8cm}+ \phi_{C}[C=1]\cdot\Theta[D=1|C=1]\Theta[E=1|C=1]\end{gather*}
Hence considering noisy $\tilde{\boldsymbol{\phi}}_{D,E}$ we have, \begin{gather*}\scalebox{0.8}{$\Upsilon_{\phi_C[0]}=(\phi_C[C=0]-\tilde{\phi}_C[C=0])\cdot(\Theta[D=0|C=0]\Theta[E=0|C=0]$}\\\scalebox{0.8}{$
+\Theta[D=0|C=0]\Theta[E=1|C=0]+\Theta[D=1|C=0]\Theta[E=0|C=0]$}\\\scalebox{0.8}{$+\Theta[D=1|C=0]\Theta[E=1|C=0])$}
\\\scalebox{0.8}{$=(\phi_C[C=0]-\tilde{\phi}_C[C=0])\cdot\Big(\Theta[D=0|C=0]\big(\Theta[E=0|C=0]+\Theta[E=1|C=0]\big)$}
\\\scalebox{0.8}{$+\Big(\Theta[D=1|C=0]\big(\Theta[E=0|C=0]+\Theta[E=1|C=0]\big)\Big)$}\\\scalebox{0.8}{$=(\phi_C[C=0]-\tilde{\phi}_C[C=0])\cdot\big(\Theta[D=0|C=0]+\Theta[D=1|C=0]\big)$}\\
\scalebox{0.8}{$\Big[\because \Theta[E=0|C=0]+\Theta[E=1|C=0] = 1\Big] $}
\\\scalebox{0.8}{$=\phi_C[C=0]-\tilde{\phi}_C[C=0]$} \numberthis \label{eq:ex1}
\\
\scalebox{0.8}{$\Big[\because \Theta[D=0|C=0]+\Theta[D=1|C=0] = 1\Big]$}
\end{gather*}
Similarly \begin{gather}\Upsilon_{\phi_C[1]}=\phi_C[C=1]-\tilde{\phi}_C[C=1] \label{eq:exp2} \end{gather}
Now using Eq. \eqref{eq:upsilon2} and summing over $\forall x \in dom(X)$\begin{gather*}\sum_x\Upsilon_{\phi_1[x]}=\sum_x(\phi_1[x]-\tilde{\phi}_1[x])\\=0\Big[\because \sum_x\phi_1[x]=\sum_x\tilde{\phi}_1[x]=1\Big] \numberthis \label{eq:Upsilon3}\end{gather*} Referring back to our example above, since $\phi_C[1]+\phi_C[0]=\tilde{\phi}_C[C=0]+\tilde{\phi}_C[C=1]$, quite trivially \begin{gather*}\phi_C[C=0]+\phi_C[C=1]=\tilde{\phi}_C[C=0]+\tilde{\phi}_C[C=1]\\\Rightarrow (\phi_C[C=0]-\tilde{\phi}_C[C=0])+(\phi_C[C=1]-\tilde{\phi}_C[C=1])=0 \end{gather*}
Thus, from Eq. \eqref{eq:twoterms} \begin{gather*}\delta_{\phi_\mathpzc{A}}=\sum_{x}{\Upsilon}_{\phi_1[x]}\pm\sum_{a,x}\delta_{\prod_{i=2}^t\phi_i[a,x]}\prod_{i=1}^t\phi_i[a,x]\\=\sum_{a,x}\delta_{\prod_{i=2}^t\phi_i[a,x]}\prod_{i=1}^t\phi_i[a,x]\\\Big[\text{From Eq. \eqref{eq:Upsilon3}} \text{ and dropping $\pm$ as we are dealing with errors}\Big]
\\\geq \delta^{min}_{\prod_{i=2}^t\phi_i[a,x]}\sum_{a,x}\prod_{i=1}^t\phi_i[a,x]\\
\Big[\delta^{min}_{\prod_{i=2}^t\phi_i[a,x]}=\min_{a,x}\Big\{\delta_{\prod_{i=2}^t\phi_i[a,x]}\Big\}\Big] \\
\geq \delta^{min}_{\prod_{i=2}^t\phi_i[a,x]} 
\\\Big[\because \text{By Lemma \ref{lemma:factor} $\phi_{\mathpzc{A}}$ is a CPD, thus } \sum_{a,x}\prod_{i=1}^t\phi_i[a,x] \geq 1\Big]
\\=min_{a,x}\Bigg\{\prod_{i=2}^t\phi_i[x,a] \sqrt{\sum_{i=2}^t\Big(\frac{\delta_{\phi_i[a,x]}}{\phi_i[a,x]}\Big)^2}\Bigg\}\end{gather*}
\begin{gather*}\\\geq min_{a,x}\Bigg\{\prod_{i=2}^t\phi_i[x,a] \sqrt{(t-1)\Big(\frac{\delta^{min}_{\phi_i[a,x]}}{\phi_i[x,a]}\Big)^2}\Bigg\}\\
[\delta^{min}_{\phi_i[x,a]}={min}_{i,a,x}\Big\{\delta_{\phi_i[a,x]}\Big\}]
\\\geq min_{a,x}\Bigg\{\sqrt{(t-1)}\frac{\delta^{min}_{\phi_i[a,x]}}{\phi^{max}_i}\prod_{i=2}^t\phi_i[a,x] \Bigg\}\\
\Big[\phi^{max}_i=\max_i\{\phi_i[a,x]\}\Big]\\\geq \delta^{min}_{\phi_i[a,x]}\sqrt{t-1}(\phi^{min}_i[a,x])^{t-2}\label{eq:secondterm}\numberthis \\\Big[ \mbox{ Assuming } \phi^{min}_i[a,x]=min_{i,a,x}\{\phi_i[a,x]\}\Big]
\end{gather*}
Now, recall from the variable elimination algorithm that during each elimination step, if $Z$ is the variable being eliminated then we the product term contains all the factors that include $Z$. For a DGM with graph $\mathcal{G}$, the maximum number of such factors is clearly $\eta=\max_{X_i}\{ \text{out-degree}(X_i) + \text{in-degree}(X_i)\}+1$ of $\mathcal{G}$, i.e., $t \leq \eta$. Additionally we have $\phi^{min}[a,x]\leq \frac{1}{d_{min}}\leq \frac{1}{2}$ where $d_{min}$ is the minimum size of $dom(Attr(\phi))$ and clearly $d_{min}\geq 2$. Since $2^{t} \geq \sqrt{t}, t \geq 2$,  under the constraint that $t$ is an integer and $\phi^{min}[a,x]\leq \frac{1}{2}$, we have \begin{gather}\delta_{\phi_\mathpzc{A}}\geq \sqrt{\eta-1}\delta^{min}_{\phi_i[a,x]}(\phi^{min}[a,x])^{\eta-2} \label{eq:lowerbound}\end{gather}
\end{proof}


\subsubsection{Upper Bound}

\begin{customthm}{4.2}  \textit{For a DGM \scalebox{0.9}{$\mathcal{N}$}, for any sum-product term of the form \scalebox{0.9}{$ \boldsymbol{\phi}_\mathpzc{A}=\sum_x\prod_{i=1}^t\phi_i$, $ t \in \{2,\cdots,n\}$} in the  VE algorithm with the optimal elimination order, \begin{equation}\small\delta_{\phi_\mathpzc{A}} \leq 2\cdot \eta\cdot d^{\kappa}\delta^{max}_{\phi_i[a,x]} \end{equation}  where \scalebox{0.9}{$X$} is the attribute being eliminated, \scalebox{0.9}{$\delta_{\phi}$} denotes the error in factor \scalebox{0.9}{$\phi$}, \scalebox{0.9}{$\kappa$} is the treewidth of \scalebox{0.9}{$\mathcal{G}$}, $d$ is the maximum attribute domain size, \scalebox{0.9}{$Attr(\phi)$} is the set of attributes in \scalebox{0.9}{$\phi$},  \scalebox{0.9}{$\small \mathpzc{A}=\bigcup_{i}^t \{Attr(\phi_i)\}/X$}, \scalebox{0.9}{$a \in dom(\mathpzc{A}), \thinspace\thinspace x \in dom(X),\thinspace \thinspace$}\scalebox{0.9}{$\phi[a,x]$} denotes that \scalebox{0.9}{$Value(Attr(\phi))\in \{a\} \wedge X=x$},  \scalebox{0.9}{$\delta^{max}_{\phi_i[a,x]}=$}
\scalebox{0.9}{$\max_{i,a,x}\{\delta_{\phi_i[a,x]}\}$ } and \scalebox{0.9}{$ \eta=\underset{X_i}{\max}\{\text{in-degree}(X_i) + \text{out-degree}(X_i)\}+1$}. }
\end{customthm}
\begin{proof}
\textbf{Proof Structure:}

The proof is structured as follows. First we compute an upper bound for a product of $t>0$ noisy factors $\tilde{\phi}_i[a,x], i \in [t]$ (Lemma \ref{lem:help3}). Next we use this lemma, to bound the error, $\delta_{\phi_\mathpzc{A}}[a]$, for the factor, $\phi_\mathpzc{A}[a], a \in dom(\mathpzc{A})$ (Eq. \eqref{eq1}). Finally we use this result to bound the total error, $\delta_{\phi_\mathpzc{A}}$, by summing over $\forall a \in dom(\mathpzc{A})$ (Eq. \eqref{eq2}).

\textbf{Step 1: Computing the upper bound of the error of a single term $\tilde{\boldsymbol{\phi}}_\mathpzc{A}[a]$, $\delta_{\boldsymbol{\phi_\mathpzc{A}}[a]}$}
    \begin{lem} For $a \in dom(\mathpzc{A}), x \in dom(X)$
    \begin{gather*}\prod_{i=1}^t\tilde{\phi_i}[a,x]\leq \prod_{i=1}^t\phi_i[a,x]+\sum_{i}\delta_{\phi_i[a,x]}\end{gather*}\label{lem:help3}\end{lem}
    \begin{proof} First we consider the base case when $t=2$. 
    
    \textbf{Base Case}:\begin{gather*}\scalebox{0.95}{$\tilde{\phi_1}[a,x]\tilde{\phi_2}[a,x]= (\phi_1[a,x]\pm\delta_{\phi_1[a,x]})(\phi_2[a,x]\pm\delta_{\phi_2[a,x]})$} \\\scalebox{0.95}{$\leq (\phi_1[a,x]+\delta_{\phi_1[a,x]})(\phi_2[a,x]+\delta_{\phi_2[a,x]})$} \\\scalebox{0.95}{$=(\phi_1[a,x]\cdot\phi_2[a,x]+\delta_{\phi_1[a,x]}(\phi_2[a,x]+\delta_{\phi_2[a,x]})+\delta_{\phi_2[a,x]}\cdot\phi_1[a,x])$}\\\scalebox{0.95}{$\leq(\phi_1[a,x]\cdot\phi_2[a,x]+ \delta_{\phi_1[a,x]}  +\delta_{\phi_2[a,x]}\phi_1[a,x])$} \\ \scalebox{0.95}{$\Big[\because (\phi_2[a,x]+\delta_{\phi_2[a,x]}) \leq 1 \text{ as } \tilde{\phi_i}[a,x]\mbox{ is still }$}\\\scalebox{0.95}{$\mbox{a valid probability distribution}\Big]$} \\\scalebox{0.95}{$\leq \phi_1[a,x]\cdot\phi_2[a,x]+\delta_{\phi_1[a,x]}+\delta_{\phi_2[a,x]}$}\\ \scalebox{0.95}{$\Big[\because \phi_1[a,x]<1\Big]$} \end{gather*}
    \textbf{Inductive Case:}
    
    Let us assume that the lemma holds for $t=k$. Thus we have \begin{gather*}\prod_{i=1}^{k+1}\tilde{\phi}_i[a,x]=\prod_{i=1}^k\tilde{\phi_i}[a,x]\cdot\tilde{\phi}_{k+1}[a,x]\\ \leq (\prod_{i=1}^k\phi_i[a,x]+\sum_{i}\delta_{\phi_i[a,x]})\cdot(\phi_{k+1}[a,x]+\delta_{\phi_{k+1}}[a,x])\\\leq \prod_{i=1}^{k+1}\phi_i[a,x] + \sum_{i}\delta_{\phi_i[a,x]}\cdot (\phi_{k+1}[a,x]+\delta_{\phi_{k+1}}[a,x])  \\\hspace{3cm}+\delta_{\phi_{k+1}}[a,x] \prod_{i=1}^{k}\phi_i[a,x]
    \\\leq \prod_{i=1}^{k+1}\phi_i[a,x] + \sum_{i=1}^{k+1}\delta_{\phi_i[a,x]} +\delta_{\phi_{k+1}}[a,x] \prod_{i=1}^{k}\phi_i[a,x]\\ [\because (\phi_{k+1}[a,x]+\delta_{\phi_{k+1}[a,x]}) \leq 1 \text{ as }  \tilde{\phi}_{k+1}[a,x]\mbox{ is still}\\\mbox{ a valid probability distribution} ] 
    \\\leq \prod_{i=1}^{k+1}\phi_i[a,x] + \sum_{i=1}^{k+1}\delta_{\phi_i[a,x]} [\because \forall i ,\phi_i[a,x]\leq 1]\end{gather*}
  Hence, we have \begin{gather*}\prod_{i=1}^t\tilde{\phi_i}[a,x]\leq \prod_{i=1}^t\phi_i[a,x]+\sum_{i}\delta_{\phi_i[a,x]}\end{gather*}   \end{proof}
 Next, we compute the error for the factor, $\boldsymbol{\phi}_\mathpzc{A}[a], a \in dom(\mathpzc{A})$ as follows
    \begin{gather*}
    \delta_{\boldsymbol{\phi}_A[a]}=\Big|\sum_x\prod_{i=1}^t\phi_{i}[a,x]-\sum_x\prod_{i=1}^t\tilde{\phi_{i}}[a,x]\Big|\\=\Big|\sum_x\prod_{i=1}^t\phi_{i}[a,x]-\phi_{1}[a,x]\prod_{i=2}^t\tilde{\phi_{i}[a,x]}\Big|
    \\= \Big|\sum_x\prod_{i=1}^t\phi_{i}[a,x]-\tilde{\phi_{1}}[a,x]\prod_{i=2}^t(\phi_{i}[a,x]\pm\delta_{\phi_{i}[a,x]})\Big|
    \\\leq\Big| \sum_x\Big(\prod_{i=1}^t\phi_{i}[a,x]-\tilde{\phi}_{1}[a,x](\prod_{i=2}^t\phi_{i}[a,x]+\sum_{i=2}^t\delta_{\phi_{i}[a,x]})\Big)\Big|\\\Big[ \mbox{Using Lemma }\ref{lem:help3}\Big] \\\leq \Big|\sum_x\Big((\phi_{1}[a,x]-\tilde{\phi}_{1}[a,x])\prod_{i=2}^t\phi_{i}[a,x] \\\hspace{2cm}+\tilde{\phi}_{1}[a,x]\sum_{i=2}^t\delta_{\phi_{i}[a,x]}\Big)\Big|\\\leq\Big| \sum_x\Big((\phi_{1}[a,x]-\tilde{\phi_1}[a,x])\prod_{i=2}^t\phi_{i}[a,x] \\\hspace{2cm}+\eta\tilde{\phi_{1}}[a,x]\delta^{*}_{\phi_{i}[a,x]}\Big)\Big|\end{gather*}\begin{gather*}
    \Big[\because  t \leq \eta \text{ and assuming } \delta^{*}_{\phi_{i}[a,x]}=\max_{i,x}\{\delta_{\phi_i[a,x]}\}\Big]\\=\Big|\sum_x(\phi_{1}[a,x]-\tilde{\phi}_{1}[a,x])\prod_{i=2}^t\phi_{i}[a,x] \\\hspace{2cm}+\eta\delta^{*}_{\phi_{i}[a,x]}\sum_x\tilde{\phi}_{1}[a,x]\Big|\\\leq \sum_x\Big|\phi_1[a,x]-\tilde{\phi}_1[a,x]\Big|+\eta\delta^{*}_{\phi_i[a,x]}\sum_x\tilde{\phi}_1[a,x]  \numberthis \label{eq1}\\\Big[\because \phi_i[a,x] \leq 1\Big] \end{gather*}
    
    \textbf{Step 2: Computing the upper bound of the total error 
    $\delta_{\boldsymbol{\phi}_\mathpzc{A}}$}
    
    Now summing over $\forall a \in dom(\mathpzc{A})$,
    \begin{gather*}\delta_{\boldsymbol{\phi}_\mathpzc{A}}=\sum_a\delta_{\boldsymbol{\phi}_\mathpzc{A}[a]}\\\leq \sum_a\Big(\sum_x|\phi_1[a,x]-\tilde{\phi}_1[a,x]|+\eta \delta^{*}_{\phi_i[a,x]}\sum_x\tilde{\phi}_1[a,x] \Big) \\\text{[From Eq. \eqref{eq1}]}\\ = \delta_{\phi_1} + \eta \delta^{max}_{\phi_i{[a,x]}} \sum_a\sum_x\tilde{\phi}_1{[a,x]} \Big{[}\delta^{max}_{\phi_i[a,x]}=\max_a
    \{\delta^{*}_{\phi_i[a,x]}\}\Big{]}
    \end{gather*}
    Now by Lemma \ref{lemma:treewidth}, maximum size of $\mathpzc{A} \cup X$ is given by the treewidth of the DGM, $\kappa$. Thus from the fact that $\phi_1$ is a CPD (Lemma \ref{lemma:factor}),  we observe that $\sum_a\sum_x\tilde{\phi}_1[a,x]$ is maximized when $\phi_1[a,x] $ is of the form $P[A'|\mathbf{A}], A' \in \mathpzc{A} \cup X, |A'|=1, \mathbf{A}= (\mathpzc{A} \cup X)/A' $ and is upper bounded by $d^{\kappa}$ where $d$ is the maximum domain size of an attribute. \begin{gather*}\delta_{\boldsymbol{\phi}_\mathpzc{A}}\leq \delta_{\phi_1}+\eta d^{\kappa}\delta^{max}_{\phi_i[a,x]}\\ [\text{By lemma \ref{lemma:treewidth} and that } \phi_1 \mbox{ is CPD from lemma \ref{lemma:factor}}]\\\mbox{ where $\kappa$ is the treewidth of $\mathcal{G}$ an}\\\mbox{ $d$ is the maximum domain size of an attribute}\\\leq 2\cdot \eta \cdot d^{\kappa}\delta^{max}_{\phi_i[a,x]} \numberthis \label{eq2} \Big[\because \delta_{\phi_1} \leq \eta\cdot d^{\kappa}\delta^{max}_{\phi_i[a,x]}\Big]\end{gather*}
    
  \end{proof}


\vspace{-1cm}\section{Related Work}\label{app:relatedWork}
In this section, we review  related literature. 
There has been a steadily growing amount work in differentially private machine learning models for the last couple of years.
We list some of the most recent work in this line (not exhaustive list).
\cite{SGD1,SGD2,SGD3} address the problem  of differentially private SGD. The authors of \cite{dpem} present an algorithm for differentially private expectation maximization. In \cite{Mestimator} the problem of differentially private M-estimators is addressed.
Algorithms for performing expected risk minimization under differential privacy has been proposed in \cite{dperm1,dperm2}. In \cite{Bayes9} two differentially private subspace clustering algorithms are proposed.
\par There has been a fair amount of work in differentially private Bayesian inferencing and related notions \cite{Bayes1,Bayes2,Bayes3,Bayes4,Bayes5,Bayes6,Bayes7,Bayes10,Bayes11,Bayes12,Bayes13,Bayes14,Bayes15,Data_PAC,privbayes, MRF}. In \cite{Bayes7} the authors present a solution for DP Bayesian learning in a distributed setting, where each party only holds a subset of the data a single
sample or a few samples of the data. In \cite{Data_PAC} the authors show that a
 data-dependent prior learnt under $\epsilon$-DP yields a valid PAC-Bayes bound. The authors in \cite{DP-Prob} show that probabilistic inference over differentially private measurements to derive posterior distributions over the data sets and model parameters can potentially improve accuracy. An algorithm to learn an unknown probability distribution over a discrete population from random samples under $\epsilon$-DP is presented in \cite{Ilias}.  In \cite{Bayes6} the authors  present a method
for private Bayesian inference in exponential families that learns from sufficient statistics. The authors of \cite{Bayes2} and \cite{Bayes1} show that posterior sampling gives differential privacy "for free" under certain assumptions. In \cite{Bayes3} the authors show that Laplace mechanism based alternative for "One Posterior Sample" is as asymptotically efficient as
non-private posterior inference, under general assumptions. A R\'enyi differentially private posterior sampling algorithm is presented in \cite{Bayes5}. \cite{Bayes10} proposes a differential private Naive Bayes classification algorithm for data streams. \cite{Bayes4} presents algorithms for private Bayesian inference on probabilistic graphical models.  In \cite{Bayes12}, the authors  introduce a general privacy-preserving framework for Variational Bayes. An expressive framework for writing and verifying differentially private Bayesian machine learning algorithms is presented in \cite{Bayes14}. The problem of learning discrete, undirected graphical models in a differentially private way is studied in \cite{Bayes15}. \cite{Bayes11} presents a general method for privacy-preserving Bayesian inference in Poisson factorization. In \cite{MRF}, the authors consider the problem of learning Markov Random Fields under differential privacy. In \cite{Bayes4} the authors propose algorithms for private Bayesian inference on graphical models. However, their proposed solution does not add data-dependent noise. In fact their proposed algorithms (Algorithm 1 and Algorithm 2 as in \cite{Bayes4}) are essentially the same in spirit as our baseline solution \textit{D-Ind}. Moreover, some proposals from \cite{Bayes4} can be combined with \textit{D-Ind}; for example to ensure mutual consistency, \cite{Bayes4} adds Laplace noise in the Fourier domain  while \textit{D-Ind} uses techniques of \cite{Consistency}. \textit{D-Ind} is also identical (\textit{D-Ind} has an additional consistency step) to an algorithm used in \cite{privbayes} which uses DGMs to generate high-dimensional data.
\par A number of data-dependent differentially private algorithms have been proposed in the past few years. \cite{hist,hist2,AHP,DPcube} outline data-dependent mechanisms for publishing histograms.  In \cite{data1} the authors construct an estimate of the dataset by building differentially private kd-trees. MWEM \cite{MWEM} derives estimate of the dataset iteratively via multiplicative weight updates. In \cite{DAWA} differential privacy is achieved by adding data and workload dependent noise. \cite{DataAshwin} presents a data-dependent differentially private algorithm selection technique. \cite{dataTheory1,dataTheory2} present two general data-dependent differentially private mechanisms. Certain data-independent mechanisms attempt to find a better set of measurements in support of a given workload. One of the most prominent technique is the matrix mechanism framework \cite{DPMatrix2,DPMatrix1} which formalizes the measurement selection problem as a rank-constrained SDP. Another popular approach is to employ a hierarchical strategy \cite{H1,H2,H3}. \cite{con1,con2,con3,con4,con5,Consistency} propose techniques for marginal table release. 
\subsection{Evaluation Cntd.} \label{evaluation:cntd}
\subsubsection{Data sets}
As mentioned in Section 5, we evaluate our algorithm on the following four DGMs.

 (1) \textbf{Asia}: Number of nodes -- 8; Number of arcs -- 8; Number of parameters -- 18 

(2) \textbf{Sachs}: Number of nodes -- 11; Number of arcs -- 17; Number of parameters -- 178

(3) \textbf{Child}: Number of nodes -- 20; Number of arcs -- 25; Number of parameters -- 230

(4) \textbf{Alarm}: Number of nodes -- 37; Number of arcs -- 46; Number of parameters -- 509 

The error analysis for the data sets Asia and Alarm are presented in Fig. \ref{fig:cntd}.
\begin{figure*}[h]
 \begin{subfigure}[b]{0.4\linewidth}
    \centering    \includegraphics[width=\linewidth]{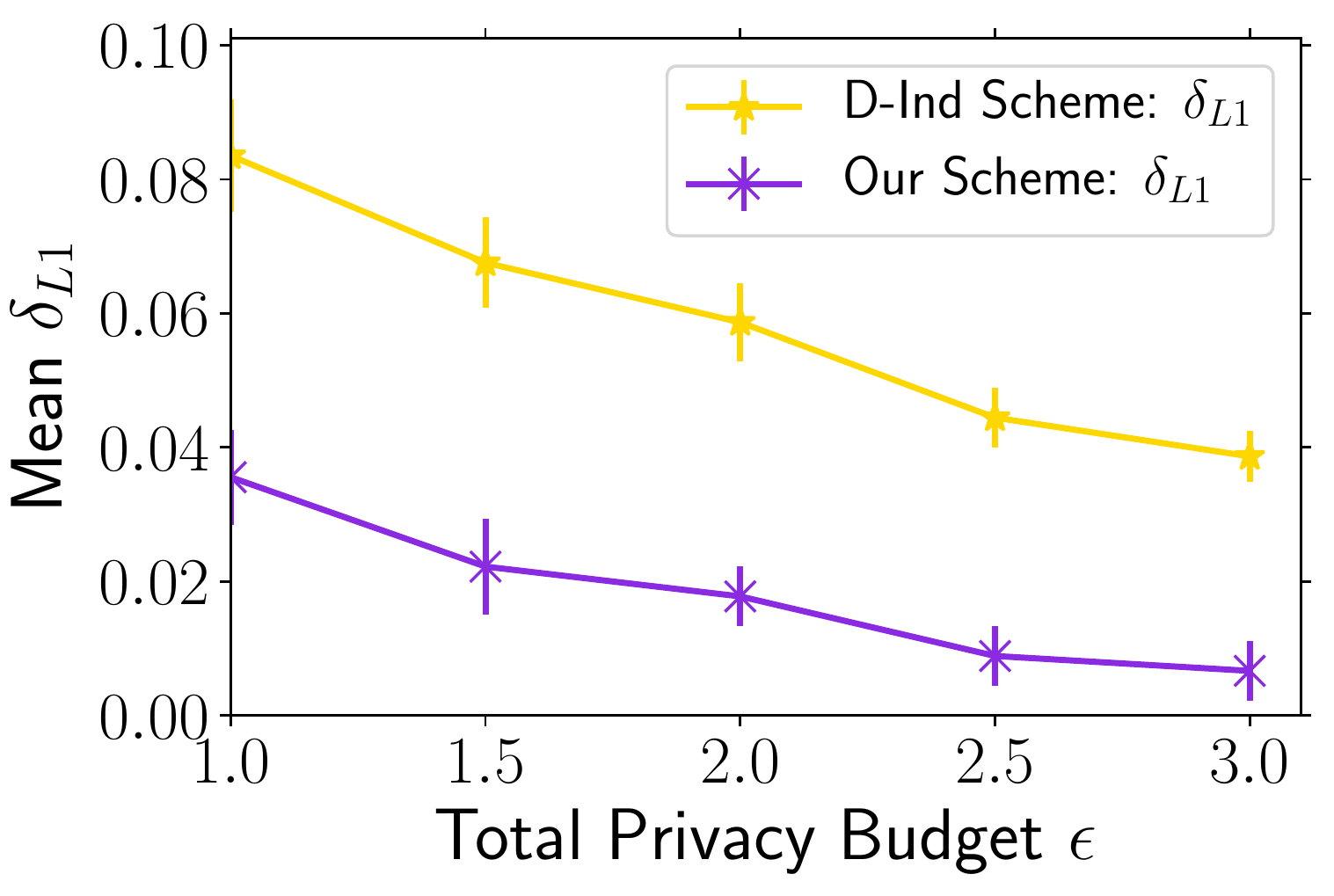} \label{fig:Para:Asia1}
    \caption{Asia: Parameters $\delta_{L1}$}
   \end{subfigure}
   \begin{subfigure}[b]{0.4\linewidth}
     \centering    \includegraphics[width=\linewidth]{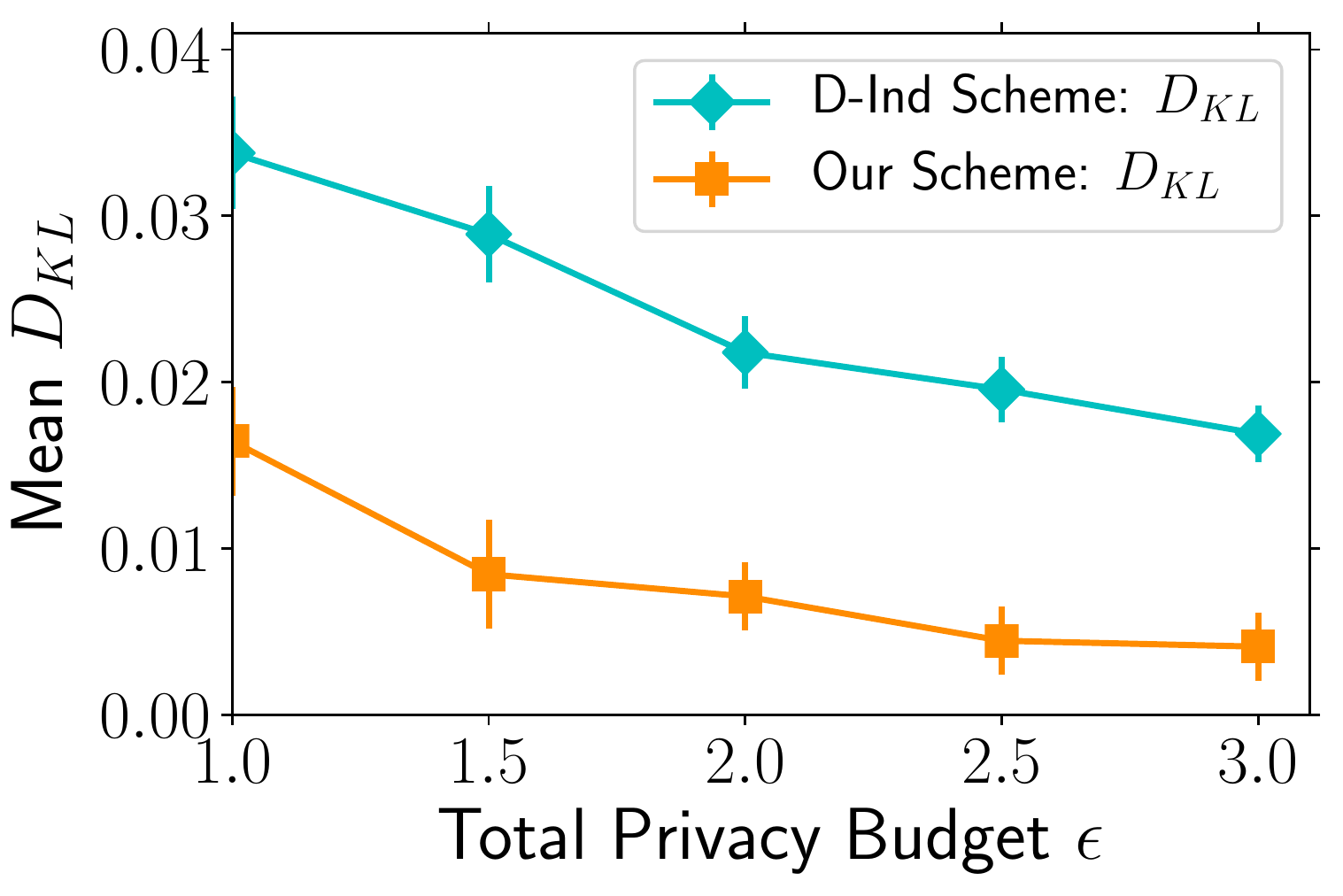} \label{fig:Para:Asia2}
        \caption{Asia: Parameters  $D_{KL}$}
       \end{subfigure}
        \\
      \begin{subfigure}[b]{0.4\linewidth}
    \centering    \includegraphics[width=\linewidth]{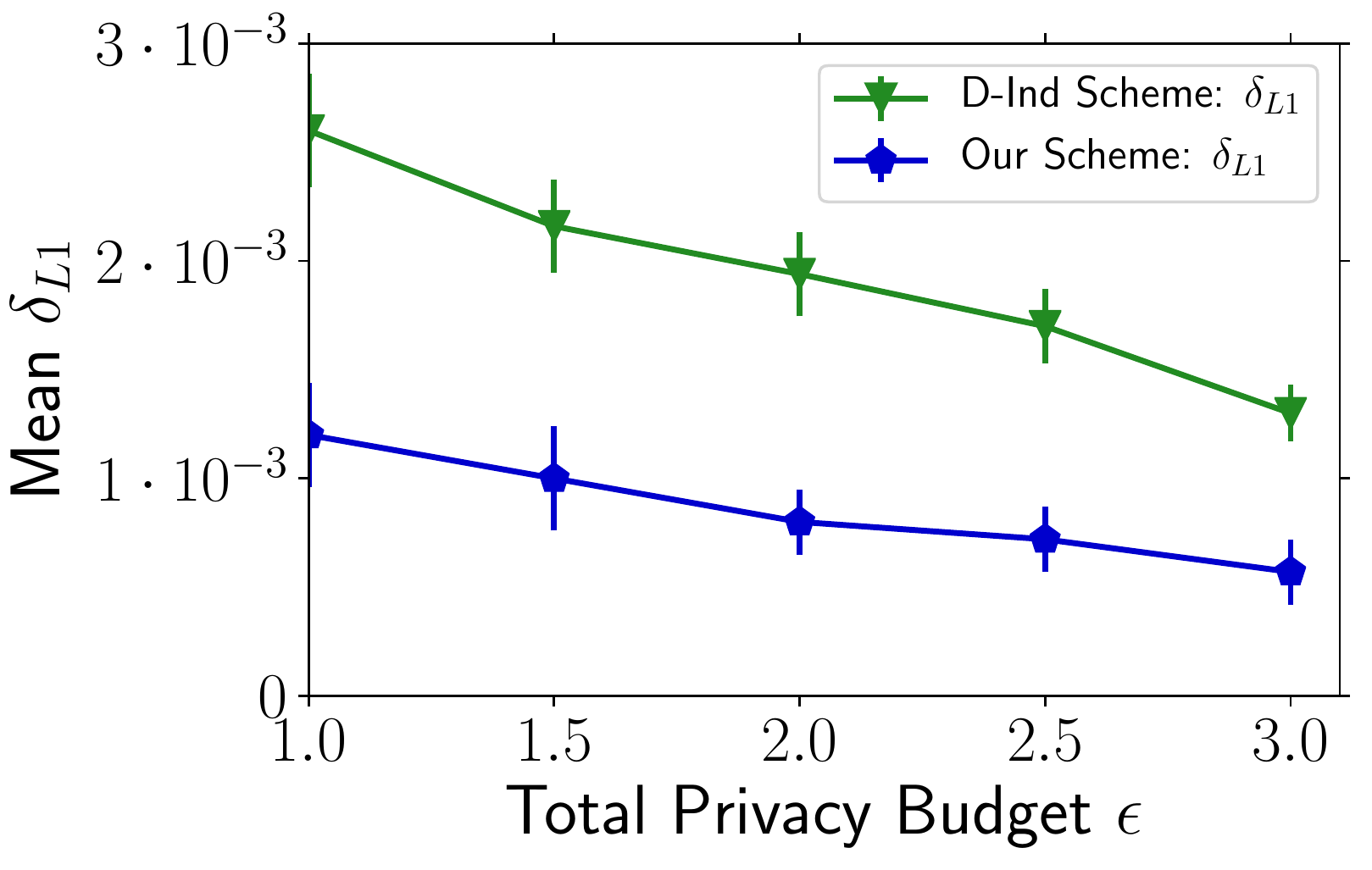} \label{fig:Inf:Asia:L1:1}
      \caption{Asia: Inference $\delta_{L1}$ }
\end{subfigure}
\begin{subfigure}[b]{0.4\linewidth}
    \centering    \includegraphics[width=\linewidth]{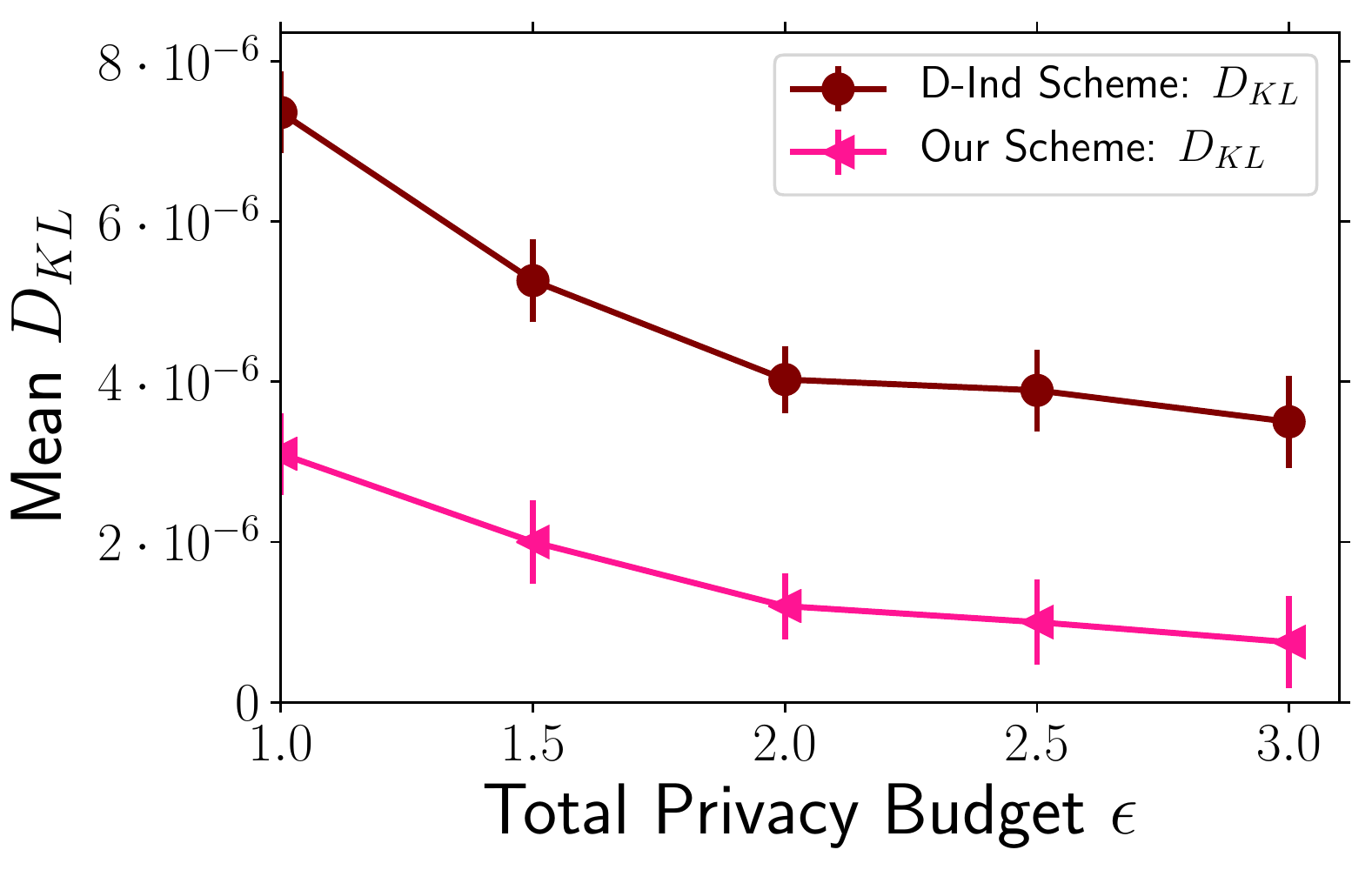}  \label{fig:Inf:Asia:L1:2}
        \caption{Asia: Inference $D_{KL}$}
        \end{subfigure} 
        \begin{subfigure}[b]{0.4\linewidth}
    \centering    \includegraphics[width=\linewidth]{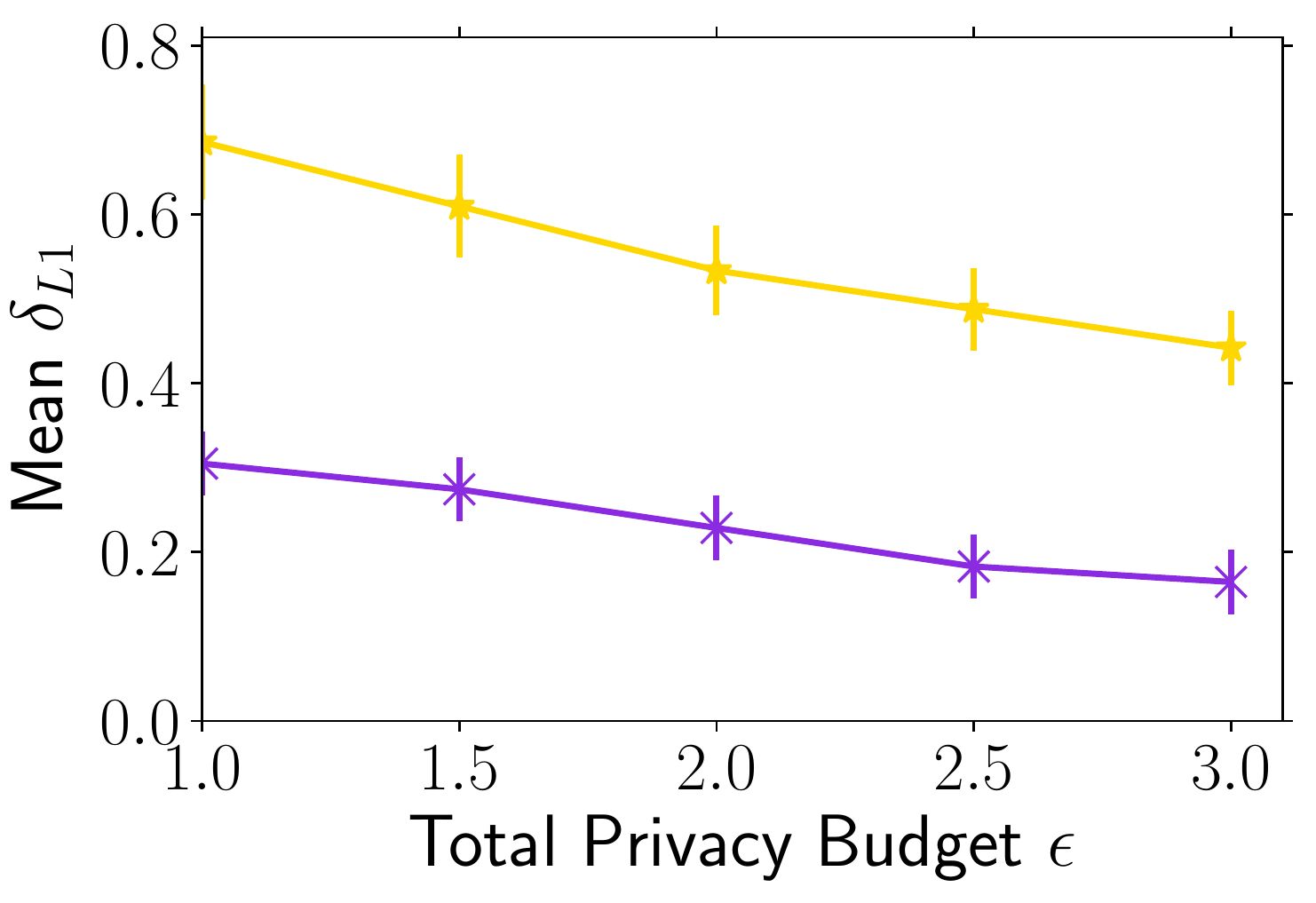}
     \caption{Alarm: Parameters $\delta_{L1}$}
    \label{fig:Inf:Asia:KL:1}
    \end{subfigure}
  \begin{subfigure}[b]{0.4\linewidth}
     \centering    \includegraphics[width=\linewidth]{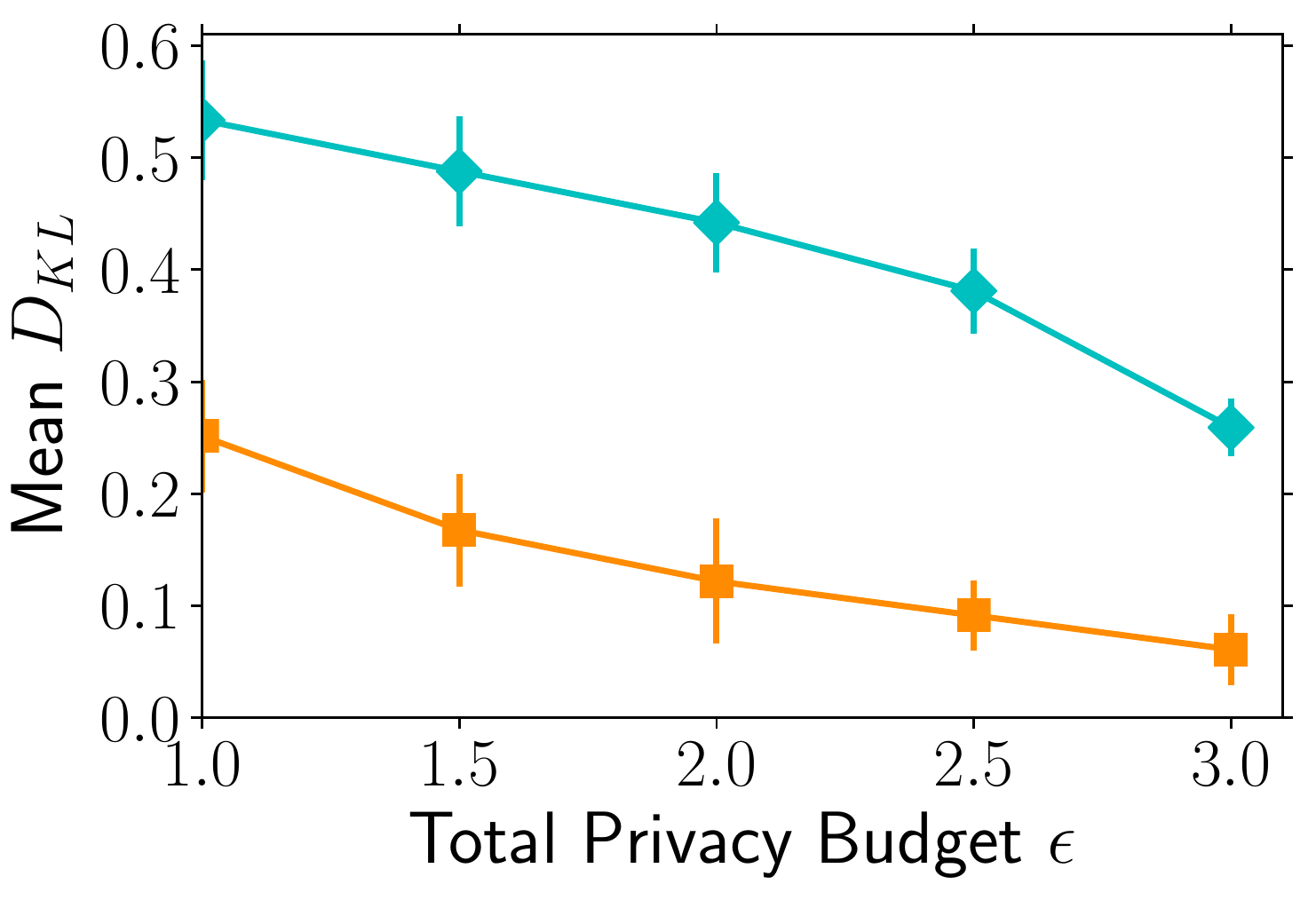}
       \caption{Alarm: Parameters  $D_{KL}$}
        \label{fig:Inf:Asia:KL:2}
    \end{subfigure}
    \\
    \begin{subfigure}[b]{0.4\linewidth}
    \centering    \includegraphics[width=1\linewidth]{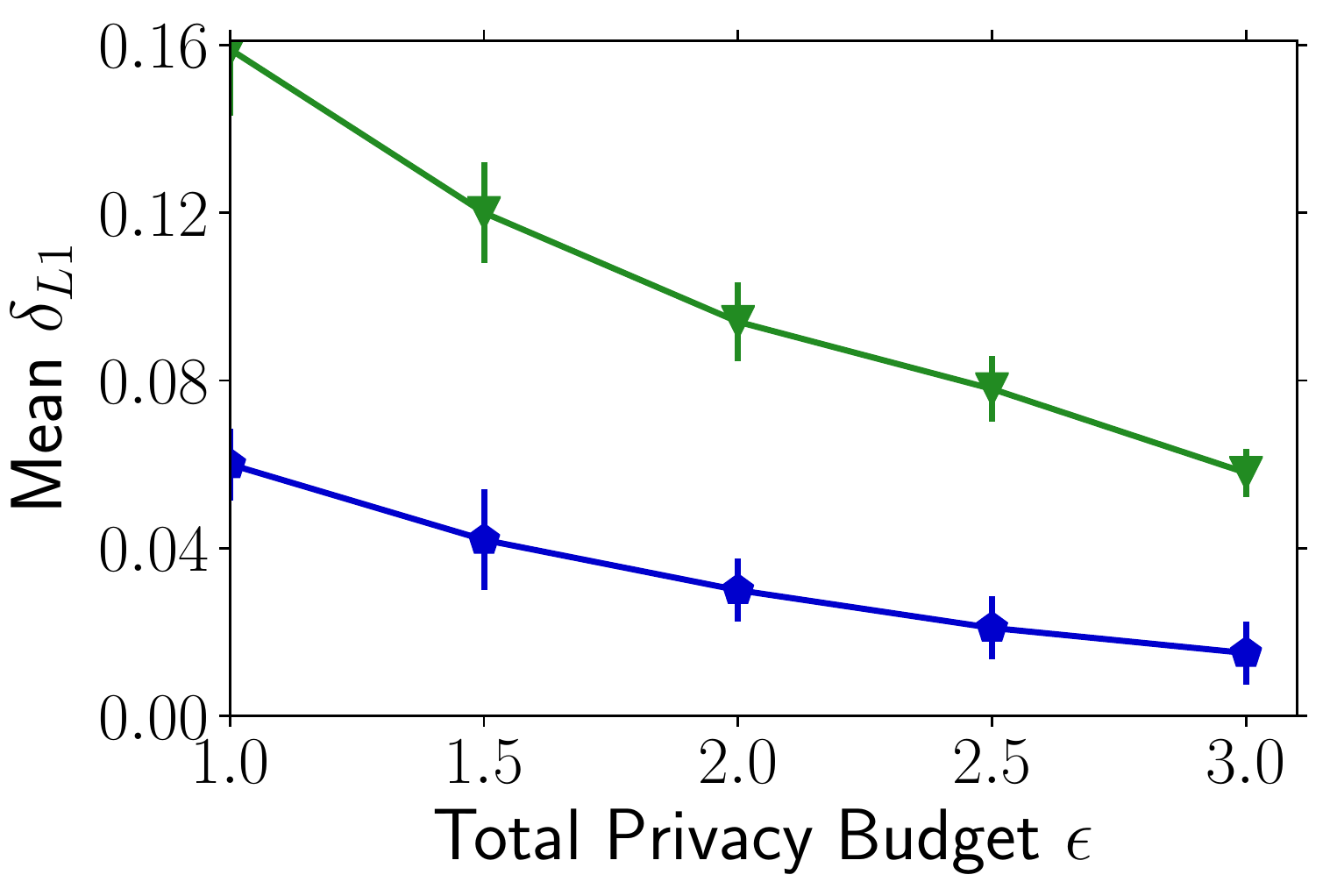} \label{fig:Para:Sachs:2}
        \caption{Alarm: Inference $\delta_{L1}$}
       \end{subfigure}
            \begin{subfigure}[b]{0.4\linewidth}
    \centering    \includegraphics[width=1\linewidth]{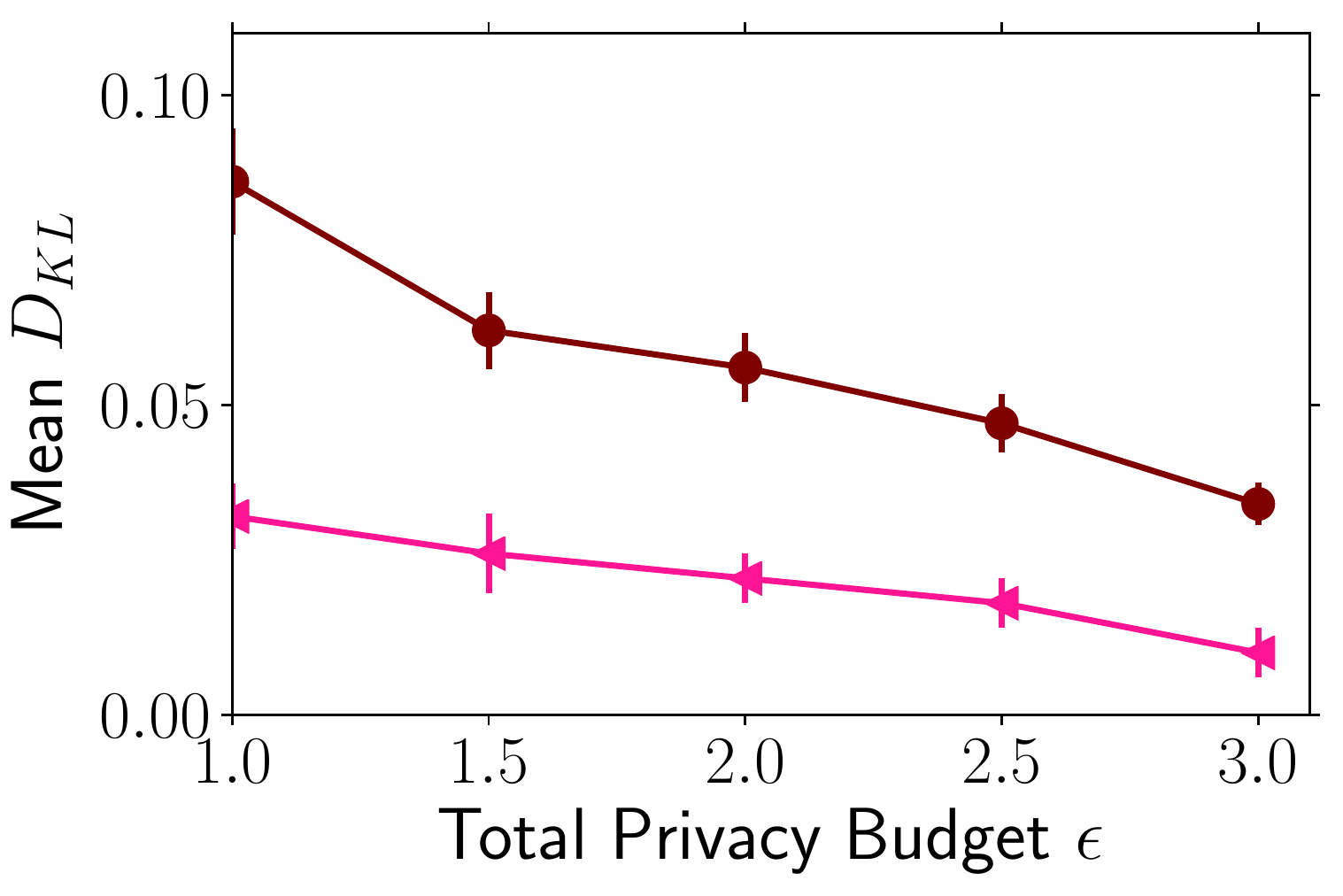}
         \caption{Alarm: Inference $D_{KL}$}
        \label{fig:Inf:Sachs:2}\end{subfigure}
         \caption{Parameter and Inference (Marginal and Conditional) Error Analysis Cntd. }\label{fig:cntd}
        \end{figure*}

\end{document}